%% file: article.tex
\documentclass[hidelinks,onefignum,onetabnum]{siamart250211}

\input{ex_shared}

\ifpdf
\hypersetup{
  pdftitle={Generalization Bounds for Equivariant Networks on Markov Data},
  pdfauthor={Hui Li, Zhiguo Wang, Bohui Chen, and Li Sheng}
}
\fi

\begin{document}

\maketitle

\begin{abstract}
Equivariant neural networks play a pivotal role in analyzing datasets with symmetry properties, particularly in complex data structures. However, integrating equivariance with Markov properties presents notable challenges due to the inherent dependencies within such data. Previous research has primarily concentrated on establishing generalization bounds under the assumption of independently and identically distributed data, frequently neglecting the influence of Markov dependencies.  In this study, we investigate the impact of Markov properties on generalization performance alongside the role of equivariance within this context. We begin by applying a new McDiarmid's inequality to derive a generalization bound for neural networks trained on Markov datasets, using Rademacher complexity as a central measure of model capacity. Subsequently, we utilize group theory to compute the covering number under equivariant constraints, enabling us to obtain an upper bound on the Rademacher complexity based on this covering number. This bound provides practical insights into selecting low-dimensional irreducible representations, enhancing generalization performance for fixed-width equivariant neural networks.
\end{abstract}

\begin{keywords}
Equivariant neural networks; generalization bound; covering number; Markov data.
\end{keywords}

\begin{MSCcodes}
68Q25
\end{MSCcodes}

\section{Introduction}
In recent years, deep neural networks have made significant strides across diverse applications in artificial intelligence and data science \cite{lecun2015deep, he2021deep, scarabosio2022deep, chen2024deep}. A central challenge within statistical learning theory has been understanding the generalization ability of these models \cite{vapnik1998statistical,xiao2022stability,bartlett2017spectrally,neyshabur2017pac,xiao2023pac,koltchinskii2002empirical,truong2022generalization}. Vapnik's foundational work in \cite{vapnik1998statistical} introduced methods for bounding generalization error based on the VC dimension of the function class. Later, \cite{koltchinskii2002empirical} expanded on this by providing probabilistic upper bounds for the generalization error in complex model combinations, including deep neural networks. More recent advancements include \cite{bartlett2017spectrally}, which introduced a spectral normalized margin bound, and \cite{neyshabur2017pac}, which utilized PAC-Bayes theory to derive similar bounds.

While much of this research focuses on i.i.d. datasets, many practical applications involve correlated data, as seen in speech, handwriting, gesture recognition, and bioinformatics. Time-series data with stationary distributions, such as those from Markov Chain Monte Carlo (MCMC), finite-state random walks, or graph-based random walks, serve as key examples. These dependencies necessitate new approaches for evaluating the generalization error in networks trained on non-i.i.d. data. Recent studies, such as \cite{truong2022generalization}, offer probabilistic upper bounds for neural networks trained on Markov data, extending earlier work by \cite{koltchinskii2002empirical}. This is crucial for understanding how neural networks perform in real-world scenarios where data exhibits inherent dependencies.

\begin{figure}[t]
  \centering
  \includegraphics[width=0.6\linewidth]{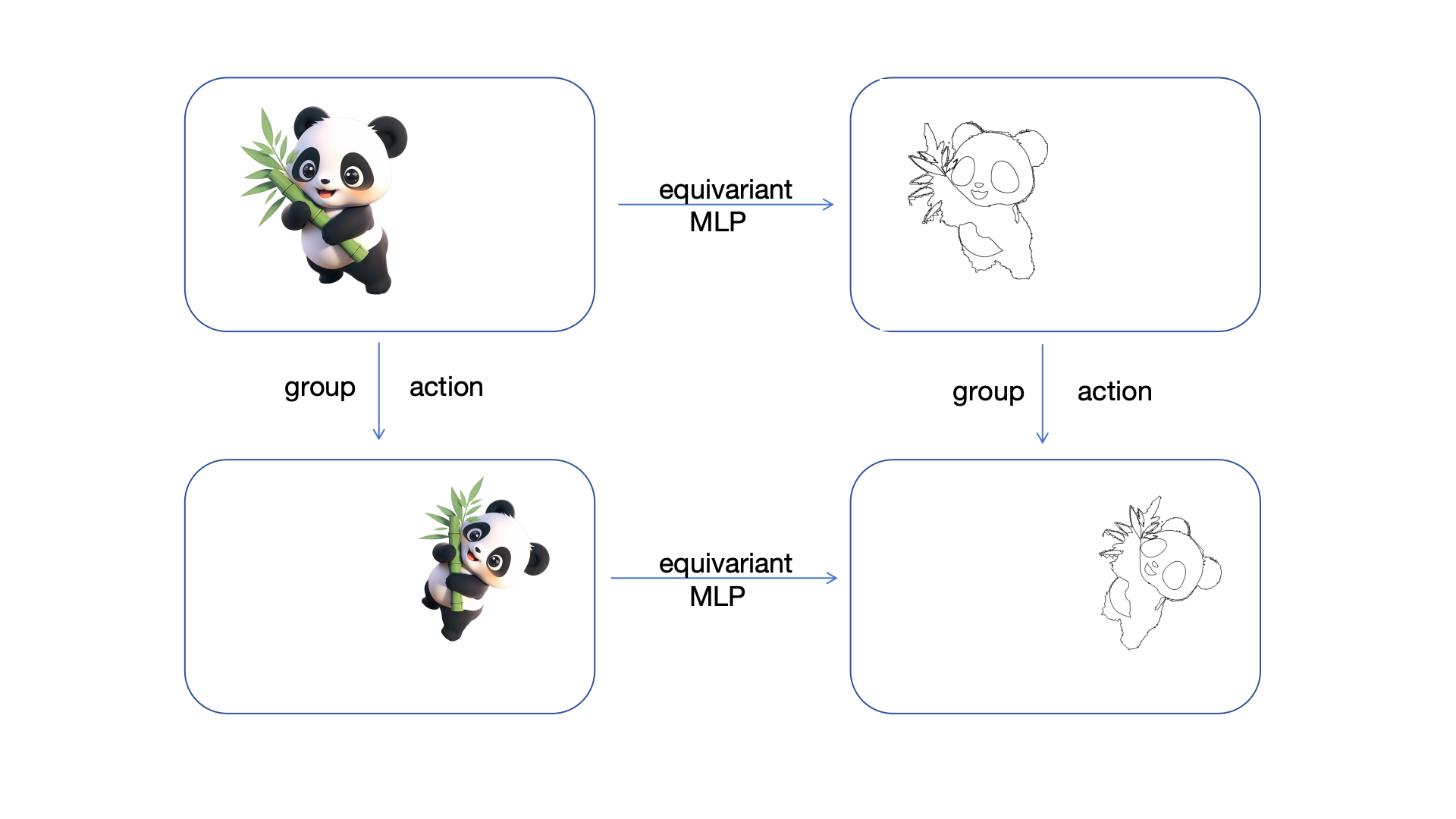}
  \caption{An equivariant neural network preserves transformations in the input, generating consistent outputs under rotation and scaling}
  \label{fig:1}
\end{figure}
 In this context, equivariant neural networks have emerged as a promising approach, leveraging symmetries in data to make models more efficient and robust. Fig.~\ref{fig:1} illustrates how an equivariant network preserves the relationship between input and output under rotational and scaling transformations, making it particularly valuable in fields like fluid dynamics \cite{wang2020incorporating}, molecular dynamics \cite{atz2021geometric}, particle physics \cite{bogatskiy2020lorentz}, robotics \cite{zhu2022sample, huang2022equivariant}, and reinforcement learning \cite{van2020mdp, wang2022so}. A key question here is whether such inductive biases can improve generalization. For finite groups, \cite{sokolic2017generalization} shows that invariant algorithms can achieve generalization errors up to $\sqrt{|G|}$ smaller than those of non-invariant algorithms. This idea is further explored by \cite{elesedy2021provably} and \cite{behboodi2022pac}, particularly for i.i.d. data.

Despite the growing interest in equivariant networks, their theoretical understanding under Markov data remains limited. This paper addresses this gap by deriving a novel generalization bound for equivariant networks trained on Markov data, demonstrating that incorporating equivariance can significantly enhance generalization, even when data exhibits time-dependent characteristics.

To investigate these bounds, we face two key challenges:

\textbf{(1) Derive a Generalization Gap Bound Based on the Empirical Radem-acher Complexity for Markov Datasets.} In standard setups with i.i.d. data, empirical Rademacher complexity serves as an effective measure of a model’s capacity, helping to establish bounds on generalization error. For example, in \cite{mohri2018foundations}, the generalization bound for i.i.d. data takes the form:
\begin{equation}\label{eqn_gen_MLP}	
    \mathcal{R}_\gamma(f) \leq \widehat{\mathcal{R}}_\gamma(f) + 2\mathfrak{R}_S(\mathcal{F_\gamma}) + 3\sqrt{\frac{\ln(2/\delta)}{2n}},
\end{equation}
where $\mathcal{R}_\gamma(f)$ is the population risk, $\widehat{\mathcal{R}}_\gamma(f)$ is the empirical risk, and $\mathfrak{R}_S(\mathcal{F_\gamma})$ denotes the empirical Rademacher complexity of the function class.
However, this bound heavily relies on the i.i.d. assumption that do not hold in the case of Markov datasets, where dependencies between samples complicate the analysis. This issue is evident in two main aspects: first, McDiarmid's inequality, which applies under i.i.d. conditions, cannot be directly used here. Second, for i.i.d. data, we have
$$\mathbb{E}[h(z)]=\mathbb{E}_S\big[\widehat{ \mathbb{E}}_S[h(z)]\big],$$
where $\mathbb{E}[h(z)]$ is the expectation under the true data distribution and $\widehat{ \mathbb{E}}_S[h(z)]$ is the empirical expectation. However, with Markov data, $\mathbb{E}[h(z)]$ represents an expectation over the stationary distribution, whereas $\mathbb{E}_S\big[\widehat{ \mathbb{E}}_S[h(z)]\big]$ reflects expectations over the dataset, they are no longer equal.

\textbf{ (2) Derive a Tight Bound on the Empirical Rademacher Complexity for Equivariant Neural Networks.} Using the Dudley Entropy Integral formula, we can derive an upper bound on the Rademacher complexity by calculating the covering number of the neural network’s function space. In \cite{bartlett2017spectrally}, the upper bound on the covering number is obtained using Maurey’s sparsification lemma, which provides an effective estimate of the covering number for matrix products of the form $W_l X_{l-1}$, where $W_l$ denotes the weight matrix of the $l$-th layer and $X_{l-1}$ is the input to that layer. However, due to inherent symmetry properties, equivariant networks introduce additional structural constraints, expressed as:
$$
W_l \rho_{l-1}(g) = \rho_l(g) W_l,
$$
where $\rho_{l-1}(g)$ and $\rho_{l}(g)$ denote the group transformations used in the corresponding layer. This constraint effectively reduces the size of the function class, suggesting that the effective hypothesis space may be smaller and lead to improved generalization performance. Nevertheless, incorporating these structural constraints into the analysis of covering number bounds is challenging.

Indeed, addressing these two issues requires innovative adaptations of empirical process techniques and complexity measures to effectively handle both the dependencies within Markov datasets and the additional structural constraints imposed by the equivariance properties of the network.
\subsection{Main Contributions}

\begin{itemize}
    \item To derive a generalization gap bound based on the empirical Rademacher complexity for Markov datasets, we introduced a novel adaptation of McDiarmid’s inequality specifically for Markov chains and incorporated the concept of mixing time. This approach allowed us to establish a new generalization bound for neural networks trained on data generated by Markov processes, resulting in Theorem~\ref{thm 1}. Our findings highlight that when data exhibits Markov dependencies, the generalization performance is influenced by an additional term in the upper bound, which compensates for the discrepancy between the initial distribution 
$\nu$
 and the stationary distribution 
$\pi$ of the Markov chain.
\item To address the equivariance constraint on the weight matrix 
$W_l$, we applied the Peter-Weyl theorem for compact groups to represent $\rho_l(g)$ and 
$\rho_{l-1}(g)$ as block diagonal matrices composed of irreducible representations. Using Schur's lemma, we then established the block diagonal structure of 
$W_l$ that aligns with the equivariance properties. Finally, we employed Maurey's sparsification lemma to compute the covering number of the set $\mathcal{H'}_l$  capturing the reduced complexity due to the equivariant constraints. This result is formalized in Lemma~\ref{linear_covering}.
\item 
By employing the standard Dudley entropy integral, we upper-bound the empirical Rademacher complexity through the covering number of 
$L$-layer equivariant neural networks. Combining this result with Theorem~\ref{thm 1}, we establish a tight generalization bound for equivariant neural networks trained on data generated by Markov processes, as presented in Theorem~\ref{theorem_gb}. This bound provides valuable guidance for designing more effective equivariant neural networks: selecting different irreducible representations enhances generalization performance.
\end{itemize}
\subsection{Related Work}

\textbf{Equivariant Network:} Equivariant neural networks have made significant progress in recent years. Steerable CNNs \cite{cohen2022steerable} extended convolutional networks to $ O(2) $ equivariance, and further work expanded this to $ O(3) $ \cite{cohen2018spherical}. Frameworks for $ E(n) $-equivariant networks \cite{cesa2022program} and the e3nn library \cite{geiger2022e3nn} have been developed for 3D systems. Recent work has focused on affine symmetries, with studies like \cite{vlavcic2021affine} exploring affine symmetries and network identifiability, while affine-equivariant networks based on differential invariants \cite{li2024affine} and scale-adaptive networks \cite{shen2024efficient} improve versatility. Efficient learning methods and spatially adaptive networks were also proposed \cite{he2021efficient, he2022neural}.

In robotic manipulation, SE(3)-equivariant networks \cite{simeonov2022neural, huang2022equivariant} and reinforcement learning with SO(2)-equivariant models \cite{wang2022so} enhance task efficiency. Other advancements like Tax-pose \cite{pan2023tax}, SEIL \cite{jia2023seil}, and models like Gauge-equivariant transformers \cite{he2021gauge} and PDE-based CNNs \cite{smets2023pde} provide theoretical foundations.

Markov models in equivariant networks, such as hybrid CNN-Hidden Markov chains \cite{goumiri2023new}, and improvements to MDPs \cite{wang2022so}, showcase enhanced decision-making and data efficiency. Extensions in partial observability \cite{nguyen2023equivariant} strengthen the robustness of these approaches.

\textbf{Generalization error for equivariance:} 
Incorporating equivariance in neural networks has shown promise in improving generalization performance. \cite{sokolic2017generalization} demonstrated that symmetry enhances the learning ability of invariant classifiers, with further work exploring robust large-margin deep networks. \cite{elesedy2021provably1} and \cite{elesedy2022group} highlighted the generalization benefits of symmetry in kernel methods and PAC learning, respectively.

Additionally, \cite{lyle2020benefits} and \cite{bietti2019group} showed how invariance leads to better robustness and data efficiency, while \cite{elesedy2021provably} demonstrated the strict generalization advantages of equivariant models. \cite{zhu2021understanding} provided empirical evidence of these benefits, and \cite{behboodi2022pac} established a PAC-Bayesian bound for equivariant networks, quantifying their performance improvements. These studies collectively highlight that equivariance enhances both stability and generalization in neural networks.

\subsection{Outline}

The paper is organized as follows. Section~\ref{sec_pm} introduces key concepts and common notations used throughout the analysis. In Section~\ref{sec_GB}, we derive the generalization upper bound for MLPs trained on the Markov dataset. Section~\ref{sec_ENN} extends the analysis to equivariant neural networks, focusing on the derivation of generalization bounds. A critical part of this is calculating the covering number for linear models under equivariant constraints. In Section~\ref{sec:experiments}, we present experimental results validating the theoretical findings. Finally, Section~\ref{con} concludes the paper, summarizing the key contributions and implications of our work.
\section{Preliminaries and Notations}\label{sec_pm}

In this section, we establish the foundational notation and formalize key concepts that will serve as the basis for the rest of our discussion.

 \subsection{Equivariant Neural Network}
In addressing the multi-class classification problem,  we first consider a Multi-Layer Perceptron (MLP) with 
$L$ layers. The network consists of 
$L$ non-linear activation functions, denoted by $(\sigma_1,$ $\ldots, \sigma_{L})$, where each $\sigma_l$ is $ c_l $-Lipschitz continuous and  satisfies $\sigma_l(0)=0$. Each layer is associated with a weight matrix $W_l \in \mathbb{R}^{d_l\times d_{l-1}}$, where $l \in [L]$ and $[L]$ is denoted as the set $\{1,2,\ldots,L\}$. Whenever input data $x_1,\cdots,x_n\in \mathbb{R}^{d_0}$ are given, collect them as columns of a matrix $X\in \mathbb{R}^{d_0\times n}$. Let $ F_{\mathcal{W}} $ represent the function computed by the network and define the set of functions as follows:
\begin{equation}\label{Eq_NN}
 \mathcal F_X=\left\{	F_{\mathcal{W}}(X):= \sigma_LW_L(\sigma_{L-1}(W_{L-1}\cdots\sigma_1(W_1(X)),\|W_l\|_2\leq s_l,l\in [L]\right\}.
\end{equation}
Here, we use $ \|W\|_2 $ to denote the spectral norm of $ W $. In this paper, we also use $ \|W\|_F $ to refer to the Frobenius norm. Additionally, the $ (p,q) $-norm of $ W $, for $ W \in \mathbb{R}^{d \times m} $, is defined as:

\begin{equation*}
\| W\|_{p,q} := \left( \sum_{i=1}^m \left( \sum_{j=1}^d |w_{ij}|^p \right)^{\frac{q}{p}} \right)^{\frac{1}{q}}.
\end{equation*}

Traditional MLPs cannot guarantee to preserve symmetries in data, such as translation, rotation, or reflection. It encourages us to consider a more efficient equivariant neural network when dealing with structured data like images or 3D objects, in this case, spatial transformations are important. 

Equivariant neural networks are a class of neural networks designed to maintain the same transformation properties as the input data.  As shown in Figure~\ref{fig:1}, they ensure that if the input transforms (such as rotation, translation, or scaling), the output will be transformed correspondingly. Next, we provide some notations about equivariant neural networks. More details can be found in Appendix~\ref{R_representation}.

\textbf{Group representations and equivariance.}  Given a vector
space $\mathbb R^d$ and a compact group $G$, a representation is a group
homomorphism $\rho$ that maps a group element $g\in G$ to an invertible matrix $\rho(g)\in \mathbb R^{d\times d}$. Given two vector spaces $\mathbb R^{d_{l-1}}$ and $\mathbb R^{d_{l}}$ and corresponding representations $\rho_{l-1}$  and $\rho_l$, a function $f: \mathbb R^{d_{l-1}}\rightarrow \mathbb R^{d_{l}}$ is called equivariant if it commutes with the group action, namely
$$f(\rho_{l-1}(g)(x))=\rho_l(g)f(x), ~\forall ~x\in  \mathbb R^{d_{l-1}}.$$

\textbf{Equivariant neural network.}
If the weight matrix $W_l$ defined in (\ref{Eq_NN}) is equivariant w.r.t. the representations  $\rho_l(g)$ and $\rho_{l-1}(g)$ acting on its output and input, i.e.,
$$W_l\rho_{l-1}=\rho_lW_l.$$

Then we call it the equivariant neural network.  In this paper, let $ H_{\mathcal{W}} $ denote the function computed by the corresponding equivariant neural network, and the set of this functions is defined as follows
\begin{equation}\label{Eq_enn}
 \mathcal H_X=\big\{H_{\mathcal{W}}(X):=\{	F_{\mathcal{W}}(X), W_l\rho_{l-1}=\rho_lW_l\},\|W_l\|_2\leq s_l,l\in[L]\big\}.
\end{equation}

\subsection{Markov Datasets}\label{sec_Markov}
Let $\{z_i = (x_i, y_i)\}_{i=1}^{n}$ be uniformly ergodic a Markov chain on the state space $ Z = \mathcal{X} \times \mathcal{Y} $, with an initial distribution $\nu$, a transition kernel $ Q(z, dz') $, and a stationary distribution $ \pi $.

Denote by $ Q^t(z, \cdot) $ the $ t $-step transition kernel, and let $ \|\nu_1 - \nu_2\|_{\text{tv}} $ represent the total variation distance between two probability measures $\nu_1$ and $\nu_2$ on $ Z $.
Let $M<\infty$ and $\alpha\in[0,1)$. Then the transition kernel $Q$ with stationary distribution $\pi$ is called \emph{uniformly
ergodic} with $(\alpha, M)$ if one has for all $z\in Z$ that
\begin{align*}
    \|Q^t(z, \cdot) - \pi\|_{\text{tv}}\leq M\alpha^n, ~\forall~n\in \mathbb{N}.
\end{align*}
Additionally, the mixing time for some small $\epsilon>0$ is defined as
\begin{equation*}
t_{\text{mix}}(\varepsilon) := \min \left\{ t : \sup_{z \in Z} \|Q^t(z, \cdot) - \pi\|_{\text{tv}} \leq \varepsilon \right\}.
\end{equation*}
Since the Markov chain $\{z_n\}_{n\in \mathbb{N}}$ approximates
the desired distribution $\pi$, then we can always use it to compute the expectation of a function
\begin{equation*}
\mathcal S(h) = \int_Z h(z) \pi(dz).
\end{equation*} Let $P$ be the Markov operator with corresponding transition kernel $Q$, i.e.
\begin{equation*}
P h(z) = \int_{Z} h(z') \, Q(z, dz'), \quad z \in Z.
\end{equation*}
Then there exists an (absolute) $L_2$-spectral gap, if
\begin{align}\label{eqn_beta}
   \beta = \sup_{\|h\|_2=1}\|P(h)-\mathcal S(h)\|_2 < 1, 
\end{align}
where the $L_2$-spectral gap is given by $1 - \beta$ (see \cite{rudolf2011explicit}). 
\subsection{Generalization Bound}
 
 We consider a multi-class classification problem with $\mathcal{X} =\mathbb{ R}^{d_{0}}$ and $ \mathcal {Y}=\mathbb{ R}^k$.
The margin operator $\mathcal{M}:\mathbb{R}^k\times\{1,\cdots,k\}\to\mathbb{R}$
is defined as $\mathcal{M}(f(x), y) := f(x)_y -\max_{i\neq y}f(x)_i$. The margin loss function $l_\gamma:\mathcal{R}\to\mathcal{R}^+$ is given by:
\begin{equation*}
l_\gamma (t) = 
\begin{cases} 
	1 & t \leq 0, \\
	1 - \frac{t}{\gamma} & 0 < t < \gamma, \\
	0 & t \geq \gamma. 
\end{cases}
\end{equation*}
It is clear that $l_\gamma (t) \in [0, 1] $ and that $l_\gamma (\cdot)$ is $1/\gamma$-Lipschitz.
For any $\gamma>0$, we define the population margin loss for a multi-class classification function $f$ as $$\mathcal{R}_\gamma(f):=\mathbb{E}\big[l_\gamma(\mathcal{M}(f(x),y))\big].$$ 
 Given a sample $S:=\{(x_1,y_1),\ldots,(x_n,y_n)\}$, 
the empirical counterpart $\widehat{\mathcal{R}}_\gamma(f)$ of $\mathcal{R}_\gamma$ is defined as  
$$\widehat{\mathcal{R}}_\gamma:=\frac{1}{n}\sum_{i=1}^{n} l_\gamma\big(\mathcal{M}(f(x_i),y_i)\big).$$
For a function  $f:\mathbb{R}^{d_0}\to \mathbb{R}^k$ and any $\gamma>0$, we define a set of real-valued functions $\mathcal{F}_\gamma$ as
\begin{equation*}
		\mathcal{F}_\gamma:=\big\{(x,y)\to l_\gamma(\mathcal{M}(f(x),y):f\in\mathcal{F}\big\}.
\end{equation*}
Lastly, define the Rademacher complexity of $\mathcal{F}_\gamma$ over the sample $S$ as
\begin{equation*}
\mathfrak{R}_S(\mathcal{F_\gamma}) = \mathbb{E}_{\epsilon}\left[\sup_{h \in \mathcal{F}_\gamma} \frac{1}{n} \sum_{i=1}^n \epsilon_i h(x_i,y_i)\right],
\end{equation*}
where $\{\epsilon_i\}_{i=1}^n$ are i.i.d. Rademacher random variables taking values in $\{-1,+1\}$ with probability $\text{Pr}(\epsilon_i=-1)=\text{Pr}(\epsilon_i=1)=\frac{1}{2}$.

\section{Main Results}
The main results are presented in two parts. In subsection~\ref{sec_GB}, we derive an upper bound on the generalization errors for multilayer perceptrons (MLPs) trained on Markov datasets, extending the classical framework for bounding generalization errors in the context of i.i.d. datasets. Building on this foundation, Subsection~\ref{sec_ENN} focuses on deriving the Rademacher complexity for equivariant neural networks by adapting the method from \cite{bartlett2017spectrally}. We then apply the standard Dudley entropy integral to obtain a generalization bound for equivariant neural networks trained on data generated by Markov processes.
\subsection{Generalization Bound of MLP with Markov Datasets}\label{sec_GB}
In this subsection, we derive an upper bound on generalization errors for multilayer perceptrons (MLPs) trained on Markov datasets, building upon the classical approach for bounding generalization errors in the context of i.i.d. datasets. Compared to the i.i.d. case, we introduce the following lemma concerning Markov chains.
\label{sec:main}
\begin{theorem}\label{thm 1}
Let $\mathcal{F}_\gamma$ be a family of functions mapping from $Z = \mathcal{X} \times \mathcal{Y}$ to $[0,1]$. Given a fixed  uniformly ergodic Markov chain $S = (z_1, \cdots, z_n)$ of size $n$, where the elements $z_i$ are drawn from $Z$ with initial distribution $\nu$ and stationary distribution $\pi$, the following holds for all $f \in \mathcal{F}$ with probability at least $1 - \delta$ for any $\delta > 0$:
	\begin{align}\label{3}	\emph{Pr}\left(\mathop{\arg\max}\limits_i f(x)_i \neq y\right)\leq\widehat{\mathcal{R}}_\gamma(f)+2\mathfrak{R}_{S}(\mathcal{F}_\gamma)+3\sqrt{\frac{\tau_{\min}\ln(2/\delta)}{2n}}+C_n,
	\end{align}
	where
	$$C_n=\sqrt{\frac{2}{n(1-\beta)}+\frac{64}{n^2(1-\beta)^2}\left\|\frac{d\nu}{d\pi}-1\right\|_2}.$$
\end{theorem}
\begin{proof}
    See Appendix~\ref{Appendix B}.
\end{proof}

Compared with the generalization bound (\ref{eqn_gen_MLP}) derived for i.i.d. data, there are two additional terms in Theorem~\ref{thm 1}: $C_n$ and $\tau_{\min}$, both arising from the properties of the Markov datasets. Notably, the term $C_n$ appears in the upper bound to account for the difference between the initial distribution $\nu$ and the stationary distribution $\pi$ of the Markov chain.

Since the empirical Rademacher complexity $\mathfrak{R}_{S}(\mathcal{F}_{\gamma})$ involves taking the expectation over Rademacher variables and is independent of the distribution of the dataset, we can use the best-known upper bounds of empirical Rademacher complexity \cite{bartlett2017spectrally} to obtain the following result.

\begin{lemma}\label{lemma_MLP}
Let fixed nonlinearities $(\sigma_1,\ldots,\sigma_L)$ be given, where
	$\sigma_l$ is $c_l$-Lipschitz and $\sigma_l(0) = 0$. 
 Consider a uniformly ergodic Markov dataset $S = (z_1, \ldots, z_n)$ of size $n$, where each element $z_i$ is drawn from $Z$ with initial distribution $\nu$ and stationary distribution $\pi$. For $F_{\mathcal{W}}: \mathbb{R}^d \to \mathbb{R}^k$ with weight matrices $\mathcal{W} = (W_1, \ldots, W_L)$, let $d_{\max}= \max_l d_l$. Then, with probability at least $1 - \delta$, the following holds:

	\begin{equation*}
	\emph{Pr}\left(\mathop{\arg\max}\limits_i f(x)_i \neq y\right)\leq\widehat{R}_\gamma(f) + \tilde{ \mathcal{O} }\left( \frac{\|X\|_FR_{\mathcal W}} {\gamma n}\ln(2d_{\max}^2) + \sqrt{\frac{\tau_{\min}\ln(2/\delta)}{2n}}+C_n \right),
	\end{equation*}
	where $\| X \|_F = \sqrt{\sum_i \| x_i \|_2^2}$ and
		\begin{equation*}
	R_{\mathcal W} := \left( \prod_{i=1}^L c_i \|W_i\|_2 \right) \left( \sum_{i=1}^L \frac{\|W_i\|_{2,1}^{2/3}}{\|W_i \|_{2}^{2/3}} \right)^{3/2}.
	\end{equation*}
\end{lemma}	
\textbf{Remark.} In \cite{truong2022generalization}, an earlier neuron framework is employed, where the generalization bound is based on the product of the $\ell_1$-norm of the weight vectors in each layer and the Lipschitz constants of the activation functions. This differs slightly from the framework used in the Lemma~\ref{lemma_MLP}, which relies on weight matrices.

\subsection{Generalization Bound for Equivariant Neural Networks}\label{sec_ENN}

In this section, we derive the Rademacher complexity for equivariant neural networks by extending the approach in \cite{bartlett2017spectrally}. We begin by determining the covering number for the function space associated with equivariant linear models $\mathcal{H'}_l$. Using induction on network layers, we then establish a covering number bound for 
$L$-layer equivariant neural networks. Finally, applying the standard Dudley entropy integral allows us to obtain a generalization bound for equivariant neural networks trained on data generated by Markov processes.

\subsubsection{Covering Number for Linear Models with Equivariant}\label{sec_cL}

Let's first recall a few definitions concerning $\epsilon$-covers  and covering number.    
\begin{definition}[$\epsilon$-cover]
$\mathcal C$ is an $\epsilon$-cover of $ \mathcal Q$ with respect to metric $\|\cdot\|$ if for all $v'\in \mathcal Q$, there exists $v\in \mathcal C$ such that $\|v-v'\|\leq \epsilon$.
\end{definition}
\begin{definition}[Covering number]
The covering number $\mathcal N(\mathcal{Q}, \|\cdot\|, \epsilon)$ is defined as the minimum size of an $\epsilon$-cover.
\end{definition}

Since the linear functions with equivariant  are building blocks for multi-layer equivariant neural networks, we first consider the covering number of the following function space, denoted by
\begin{equation}
	\mathcal{H'}_l=\{W_lX_{l-1},W_l\rho_{l-1}(g)=\rho_{l}(g)W_l, g\in G,\|W_l\|_2\leq s_l\},
\end{equation}
where the matrix $X_{l-1}$ is the input of the $l$-th layer equivariant neural network, $\rho_l:G\to \mathbb{R}^{d_l\times d_l}$ is the representation used in the $l$-th layer.

The Peter-Weyl theorem implies that the representation 
$\rho_l$  can be decomposed into a direct sum of irreducible representations (irreps) as
\begin{equation}
\rho_l=Q_l^{-1}\left[\bigoplus_{k\in I}\bigoplus_{i=1}^{m_{l,\psi_k}}\psi_k\right] Q_{l-1},
\end{equation}
where the index set $I$ serves as an index for the set $\widehat{G}$ consisting of all non-equivalent irreps of the group $G$,
$m_{l,\psi_k}$
  is the multiplicity of the irrep 
$\psi_k$ in the representation $\rho_l$ (i.e., the number of times 
$\psi_k$ appears), and $Q_l$ is a basis transformation matrix. Each $\psi_k$
is a $\text{dim}_{\psi_k}\times \text{dim}_{\psi_k}$ matrix and the direct sum forms a block diagonal matrix with the irreps  $\psi_k$
 on the diagonal. From the decomposition of $\rho_l$, we can conclude that
 \begin{equation}\label{eq_sum}
     \sum_{k\in I}m_{l,\psi_k}\text{dim}_{\psi_k}=d_l.
 \end{equation}
 
 With this decomposition, we can parameterize equivariant networks in terms of irreps. Defining $\widehat W_l=Q_l^{-1}W_lQ_{l-1}$, the equivariant condition $W_l\rho_{l-1}=\rho_lW_l$ can be rewritten as
\begin{equation}\label{4}
	\widehat W_l\left(\bigoplus_k\bigoplus_{i=1}^{m_{l-1,\psi_k}}\psi_k\right) =\left(\bigoplus_k\bigoplus_{i=1}^{m_{l,\psi_k}}\psi_k\right)\widehat W_l.
\end{equation}
Note that, if $X_l = W_lX_{l-1}$ and $\widehat X^l=Q_l^{-1}X_l$, then we have $\widehat X^l=\widehat W_l\widehat X^{l-1}$.
It results in the following lemma.

\begin{lemma}\label{lem_eq}
	The covering number of $\mathcal{H}'_l$ equals the covering number of 
	\begin{equation*}
		\widehat{\mathcal{H}}_l=\left\{\widehat{W}_l\widehat X_{l-1},\widehat W_l\left(\bigoplus_k\bigoplus_{i=1}^{m_{l-1,\psi_k}}\psi_k\right) =\left(\bigoplus_k\bigoplus_{i=1}^{m_{l,\psi_k}}\psi_k\right)\widehat W_l, \|\widehat W_l\|_2\leq s_l\right\}.
	\end{equation*}
\end{lemma}
\begin{proof}
	Since $\widehat W_l=Q_l^{-1}W_lQ_{l-1}$ and $\widehat X_{l-1}=Q_{l-1}^{-1}X_{l-1}$, we have $\widehat W_l \widehat X_{l-1}=Q_l^{-1}W_lX_{l-1}$. Additionally,   $Q_{l-1}$ is an orthogonal matrix, this ensures a bijection between $\mathcal{H}'_l$ and $\widehat{\mathcal{H}}_l$, thus implying that they share the same covering number.
\end{proof}

The block diagonal structure of $(\bigoplus_k \bigoplus_{i=1}^{m_{l,\psi_k}} \psi_k)$ together with Schur's lemma for real-valued representations induces a similar structure on $\widehat W_l$. For instance, if the direct sum decomposition of the representation at the $l$-th layer utilizes $C$ irreducible representations, we have:

\begin{equation*}
\widehat{W}_l \left[
\overbrace{
\begin{array}{ccc}
\psi_1 & \cdots & 0 \\
\vdots &\ddots &\vdots  \\
0 &  \cdots& \psi_1 \\
\vdots & & \vdots \\
0 & \cdots & 0 \\
\vdots & & \vdots \\
0 & \cdots & 0 
\end{array}
}^{m_{l-1,\psi_1}}
\hspace{-3mm}
\begin{array}{c}
 \\
 \\
 \ddots\\
 \\
 \\
\\
\end{array}
\hspace{-3mm}
\overbrace{
\begin{array}{ccc}
0 & \cdots & 0 \\
\vdots &\ddots & \vdots \\
0 & \cdots & 0 \\
\vdots & & \vdots \\
\psi_C & \cdots & 0 \\
\vdots &\ddots & \vdots \\
0 & \cdots & \psi_C
\end{array}
}^{m_{l-1,\psi_C}}
\right] = \left[
\overbrace{
\begin{array}{ccc}
\psi_1 & \cdots & 0 \\
\vdots &\ddots & \vdots \\
0 & \cdots & \psi_1 \\
\vdots & & \vdots \\
0 & \cdots & 0 \\
\vdots & & \vdots \\
0 & \cdots & 0 
\end{array}
}^{m_{l,\psi_1}}
\hspace{-3mm}
\begin{array}{c}
 \\
 \\
 \ddots\\
 \\
\\
\\
\end{array}
\hspace{-3mm}
\overbrace{
\begin{array}{ccc}
0 & \cdots & 0 \\
\vdots & & \vdots \\
0 & \cdots & 0 \\
\vdots & & \vdots \\
\psi_C & \cdots & 0 \\
\vdots &\ddots & \vdots \\
0 & \cdots & \psi_C
\end{array}
}^{m_{l,\psi_C}}
\right] \widehat{W}_l
\end{equation*}

Similarly, we can divide $\widehat W_l$ into corresponding blocks, i.e.
\begin{equation*}
\widehat W_l=\begin{bmatrix}
    \widehat W_l(\psi_1, 1,\psi_1, 1)&\widehat W_l(\psi_1, 2,\psi_1, 1)&\cdots&\widehat W_l(\psi_C, m_{l-1,\psi_C},\psi_1, 1)\\
    \widehat W_l(\psi_1, 1,\psi_1, 2)&\widehat W_l(\psi_1, 2,\psi_1, 2)&\cdots&\widehat W_l(\psi_C, m_{l-1,\psi_C},\psi_1, 2)\\
    \vdots&&&\vdots\\
    \widehat W_l(\psi_1, 1,\psi_C,m_{l,\psi_C} )&\widehat W_l(\psi_1, 2,\psi_1, 1)&\cdots&\widehat W_l(\psi_C, m_{l-1,\psi_C},\psi_C, m_{l,\psi_C})\\
\end{bmatrix}
\end{equation*}
where $\widehat W_l(\psi_k,i, \psi_{k'}, j)$  represents the corresponding block connecting the $i$-th block $\psi_k$ in $(\bigoplus_{i=1}^{m_{l-1,\psi_k}} \psi_k)$ and  $j$-th block $\psi_{k'}$ in $(\bigoplus_{i=1}^{m_{l,\psi_{k'}}} \psi_k)$ for all $k, k' = 1, \ldots, C$, $j = 1, \ldots, m_{l, \psi_k}$, and $i = 1, \ldots, m_{l-1, \psi_{k'}}$.

Thus, we have:

\begin{equation*}
\widehat W_l(\psi_k,i, \psi_{k'}, j) \psi_k = \psi_{k'} \widehat W_l(\psi_k,i, \psi_{k'}, j).
\end{equation*}

According to Lemma~\ref{Schur_lemma}, if $\psi_k \neq \psi_{k'}$, which means that $\psi_k$ and $\psi_{k'}$ are non-isomorphic, then $\widehat W_l(\psi_k,i, \psi_{k'}, j) = 0$; if $\psi_k = \psi_{k'}$, $\widehat W_l(\psi_k, i, \psi_{k'}, j) := \widehat W_l(\psi_k, j, i)$, then we have

\begin{align}\label{eq_threetype}
\widehat{W}_l(\psi_k, j, i) =
\begin{cases}
    \lambda \textbf{I}_{\text{dim}_{\psi_k}} & \text{if } \psi_k \text{ is real type}, \\
    \left[
            \begin{array}{cc}
              a & -b \\
              b & a \\
            \end{array}
          \right]
     \otimes \textbf{I}_{\text{dim}_{\psi_k}/2} & \text{if } \psi_k \text{ is complex type}, \\
     \left[
            \begin{array}{cccc}
        a & -c & -b & -d \\
        c & a & d & -b \\
        b & -d & a & c \\
        d & b & -c & a \\
            \end{array}
          \right] \otimes \textbf{I}_{\text{dim}_{\psi_k}/4} & \text{if } \psi_k \text{ is quaternionic type}.
\end{cases}
\end{align}

 For convenience, we use $c_{\psi_k}=1,2,\text{or},4$ to denote the number of parameters in these three different forms, which can also be used to differentiate among the three types. It is clear that $\dim _{\psi_k}/c_\psi\geq 1$, so $c_{\psi_k}\leq \text{dim}_{\psi_k}$.

From the above discussion, we can conclude that $\widehat{W}_l$ is a block diagonal matrix as follows,
\begin{equation}\label{W_decomposition}
    \widehat W_l = \begin{bmatrix}
[\widehat W_l(\psi_1, j, i)] & 0 & \cdots & 0 \\
0 & [\widehat W_l(\psi_2, j, i)] & 0 & 0 \\
\vdots & \cdots & \cdots & \vdots \\
0 & \cdots & 0 & [\widehat W_l(\psi_C, j, i)]
\end{bmatrix}
\end{equation}
where each block 
\begin{equation*}   
[\widehat W_l(\psi_k, j, i)]=\begin{bmatrix}
    \widehat W_l(\psi_k, 1 ,1)&\widehat W_l(\psi_k, 1,2)&\cdots&\widehat W_l(\psi_k, 1,m_{l-1,\psi_k})\\
    \widehat W_l(\psi_k, 2,1)&\widehat W_l(\psi_k, 2, 2)&\cdots&\widehat W_l(\psi_k, 2,m_{l-1,\psi_k})\\
    \vdots&\vdots&\vdots&\vdots\\
    \widehat W_l(\psi_k, m_{l,\psi_k},1)&\widehat W_l(\psi_k, m_{l,\psi_k},2)&\cdots&\widehat W_l(\psi_k, m_{l,\psi_k}, m_{l-1,\psi_k})
\end{bmatrix}
\end{equation*} 
is an $m_{l, \psi_k} \times m_{l-1, \psi_k}$ block matrix. The term $\widehat W_l(\psi_k, j, i)$ corresponds to one of the three forms in (\ref{eq_threetype}). Next, we will demonstrate the covering number of the equivariant linear models by applying the Maurey sparsification lemma.

\begin{lemma}\label{linear_covering}
	For any $W_l\in\mathbb{R}^{d_l\times d_{l-1}}$, $X_{l-1}\in\mathbb{R}^{d_{l-1}\times n}$, we obtain that 
	\begin{equation*}
    \ln\mathcal{N}(\mathcal{H}'_l,\|\cdot\|_F,\epsilon)
    \leq \left\lceil
    \frac{\max_k c_{\psi_k} m_{l,\psi_k} 
    \| W_l \|_F^2 \| X_{l-1} \|_F^2}{\epsilon^2}
    \right\rceil
    \ln\left(2 D_l \right),
\end{equation*}
where $D_l=\sum_{k} c_{\psi_k} m_{l,\psi_k} m_{l-1,\psi_k} \dim_{\psi_k}$.
\end{lemma}

\begin{proof}
    See Appendix~\ref{Appendix C}.
\end{proof}
\textbf{Remark 3.2.1.}\label{remark1} When linear models do not satisfy equivariance, the covering number of such models has been derived in \cite{bartlett2017spectrally} and satisfies the following bound:
\begin{equation}\label{11}
\ln\mathcal{N}\left(W_lX_{l-1},\|\cdot\|_F,\epsilon\right)\leq\left\lceil\frac{\|W_l\|_F^2\|X_{l-1}\|_F^2d_l}{\epsilon^2}\right\rceil\ln(2d_ld_{l-1}).
\end{equation}
When we compare this result with the bound derived in Lemma~\ref{linear_covering}, we observe that equivariance reduces the covering number, which manifests in two key parts:
\begin{itemize}
    \item The first part involves the term 
$\sum_{k}c_{\psi_k}m_{l,\psi_k}m_{l-1,\psi_k}\text{dim}_{\psi_k}$,  where we know from (\ref{eq_sum}) that
 $\sum_{k}m_{l,\psi_k}\text{dim}_{\psi_k}=d_l$. Furthermore, since $c_{\psi_k}\leq \text{dim}_{\psi_k}$, we have the inequality
\begin{equation*}
c_{\psi_k}m_{l-1,\psi_k}\leq m_{l-1,\psi_k}\text{dim}_{\psi_k}\leq d_{l-1}.
\end{equation*}
Therefore, we conclude that
\begin{equation*}
\sum_{k}c_{\psi_k}m_{l,\psi_k}m_{l-1,\psi_k}\text{dim}_{\psi_k}\leq d_{l-1}\sum_{k}m_{l-1,\psi_k}\text{dim}_{\psi_k}\leq d_l\times d_{l-1}.
\end{equation*}
\item The second part  pertains to the term
\begin{equation*}
\max_k c_{\psi_k}m_{l,\psi_k} \| W_l \|_F^2 \| X_{l-1} \|_F^2\leq\|W_l\|_F^2\|X_{l-1}\|_F^2d_l,
\end{equation*}
where the inequality holds dues to
\begin{equation*}
\max_k c_{\psi_k}m_{l,\psi_k}\leq \max_k \dim_{\psi_k}m_{l,\psi_k}\leq\sum_k\dim_{\psi_k}m_{l,\psi_k}=d_l.
\end{equation*}
\end{itemize}

Thus, the covering number of equivariant linear models is smaller than that of traditional linear models, which contributes to a better generalization bound by leveraging the inherent symmetries in the data.

\subsubsection{Generalization Bound of Equivariant  Multi-layer Neural Networks}\label{ssec_GB}
In the previous section, we derived the covering number for a single-layer neural network. Now, we proceed to establish a generalization bound for a multi-layer neural network. The first step is to derive a covering bound for the $L$-layers equivariant neural network, as presented in the following lemma.
\begin{lemma}\label{lemma7}
	Let fixed nonlinearities $(\sigma_1,\ldots,\sigma_L)$ be given, where each
	$\sigma_l$ is $c_l$-Lipschitz and satisfies $\sigma_l(0) = 0$.  Let data matrix $X\in\mathcal R^{d_{0}\times n}$ be given, where the $n$ is the number of samples.
	Let $\mathcal H_L$ denote the family of matrices obtained by evaluating $X$ with all choices of $L$-layer equivariant neural network $H_{\mathcal W}$:
    \begin{equation}
 \mathcal H_X=\big\{H_{\mathcal{W}}(X),W_l\rho_{l-1}=\rho_lW_l,\|W_l\|_2\leq s_l,l\in[L]\big\},
\end{equation}
where $H_{\mathcal{W}}$ is defined in (\ref{Eq_enn}).
	Then for any $\epsilon>0$, we have the bound:
	\begin{equation*}
		\ln\mathcal N(\mathcal H_X,\|\cdot\|_F,\epsilon)\leq \frac{\|X\|_F^2\ln(2D)}{\epsilon^2}
		\left(\prod_{l=1}^{L}c_l^2s_l^2\right)\left(\sum_{l=1}^{L}\left(\frac{b_l}{s_l}\right)^{2/3}\right)^3,
	\end{equation*}
	where $b_l=\max_k \sqrt{c_{\psi_k}m_{l,\psi_k}} \| W_l \|_F$ and 
   $D=\max_l\sum_{k=1}c_{\psi_k} m_{l,\psi_k}m_{l-1,\psi_k}\emph{dim}\psi_k$.
\end{lemma}
\begin{proof}
     See Appendix~\ref{proof_of_thm8}.
\end{proof}

The covering number estimation derived from Lemma~\ref{lemma7} can be integrated with that of Lemma~\ref{lem_Dudley} to yield the final result.
\begin{theorem}{\label{theorem_gb}}
	Let fixed nonlinearities $(\sigma_1,\ldots,\sigma_L)$ be given, where each
	$\sigma_l$ is $c_l$-Lipschitz and satisfies $\sigma_l(0) = 0$.  Further, let margin $\gamma > 0$. Then with probability at least $1-\delta$ over uniformly ergodic Markov datasets draw of $n$ samples $S = (z_1, \ldots, z_n)$, where each $z_i$ is drawn from $Z$ with initial distribution $\mu$ and stationary distribution $\pi$, any equivariant neural network $H_{\mathcal W}:\mathbb R^d\to\mathbb R^k$
	with weight matrices $\mathcal W = (W_1, \ldots, W_L)$ obey $\|W_l\|_2\leq s_l$ satisfies the following bound:
		\begin{align*}
\emph{Pr}\left(\mathop{\arg\max}\limits_i f(x)_i \neq y\right) 
      \leq\widehat{R}_\gamma(f) + \tilde{ \mathcal{O} }\left( \frac{\|X\|_FR_{\mathcal W}} {\gamma n}\ln(2D) + \sqrt{\frac{\tau_{min}\ln(2/\delta)}{2n}}+C_n \right),
		\end{align*}	
		where 
		\begin{equation*}
		R_{\mathcal W} := \left( \prod_{j=1}^L c_j s_j\right)\left(\sum_{l=1}^{L}\left(\frac{\max_k \sqrt{c_{\psi_k}m_{l,\psi_k}} \| W_l \|_F}{s_l}\right)^{2/3}\right)^{3/2}
		\end{equation*}
  $$C_n=\sqrt{\frac{2}{n(1-\beta)}+\frac{64}{n^2(1-\beta)^2}\left\|\frac{d\nu}{d\pi}-1\right\|_2}.$$
$\beta$  and $\tau_{\min}$ are defined in (\ref{eqn_beta}) and (\ref{eqn_tau_min}), respectively.
\end{theorem}
\begin{proof}
See Appendix~\ref{proof_of_thm8}.
\end{proof}
\textbf{Remark 3.2.2}\label{remark}: Theorem~\ref{theorem_gb} provides a generalization error bound for equivariant neural networks trained on Markov data. This result extends the analysis of neural networks by considering the specific structure of equivariant models and the dependencies introduced by the Markov process, offering novel insights into their performance under these conditions.

When comparing the results of Theorem~\ref{theorem_gb} (equivariant) with Lemma~\ref{lemma_MLP} (non-equivariant), there is a significant improvement. The key distinction is that 
\begin{equation*}
 \sum_{k} c_{\psi_k} m_{l,\psi_k} m_{l-1,\psi_k} \dim_{\psi_k}\leq d^2_{\max},
 \end{equation*}
which leads to a tighter bound and enhanced generalization performance for equivariant neural networks. Furthermore, this generalization error bound can inform the design of more effective equivariant neural networks.

Specifically, given a group $G$, its irreducible representation space $\widehat{G}$ is typically known or must be determined beforehand. It is advantageous to select different irreducible representations. when different irreducible representations are chosen with $c_{\psi} = 1$, we obtain $m_{l,\psi_ k} = m_{l-1,\psi_ k} = 1$. 
In this scenario, \begin{equation*}
 \sum_{k} c_{\psi_k} m_{l,\psi_k} m_{l-1,\psi_k} \dim \psi_k= \sum_{k} m_{l,\psi_k} \dim \psi_k = d_l \ll d^2_{\max},
 \end{equation*} and  $$\max_k \sqrt{c_{\psi_k}m_{l,\psi_k}} \| W_l \|_F=\|W_l\|_F\leq\|W_l\|_{2,1},$$ resulting in optimal generalization performance.

\section{Experimental Results}
\label{sec:experiments}
To systematically investigate the factors affecting generalization in equivariant neural networks, we conducted experiments with the following objectives:

\textbf{1. Markov Property and Generalization:} To analyze how the Markov property in the data impacts generalization performance.

\textbf{2. Equivariance and Generalization:} To assess the generalization improvement introduced by equivariant architectures.

\textbf{3. Multiplicity in Irreducible Decomposition:} To explore the effect of multiplicities in irreducible representations on generalization.

Firstly, following \cite{behboodi2022pac}, we generated synthetic datasets based on a high-dimensional torus $\mathcal{T}^D$, ensuring that the data exhibits rotational symmetry. Additionally, based on the specific requirements of this paper, we incorporated the Markov property into the datasets. As a result, two distinct types of datasets were created:

\textbf{Markov Dataset:} The data is generated through a Markov process, where each data point’s angle is obtained by applying a small random perturbation to the angle of the previous time step, introducing temporal dependencies. The data points are embedded in a high-dimensional torus $\mathcal{T}^D$ (where $D$ is a positive integer determined based on experimental requirements) and are further perturbed with small amounts of noise. After perturbation, each point is normalized to ensure that it remains on the torus. This process enforces local correlations between consecutive points while preserving the dataset’s overall rotational symmetry. The class labels are randomly assigned as either 0 or 1, making classification independent of the data’s geometric properties.

\textbf{Independent Dataset:} Each data point is generated independently, with its angles sampled randomly without any dependency on previous points. Similar to the Markov dataset, the points lie on a high-dimensional torus $\mathcal{T}^D$ and undergo random perturbations, followed by normalization to maintain their structure. However, since each sample is drawn independently, there is no sequential dependency among the data points. The class labels are randomly assigned as either 0 or 1, ensuring that classification remains independent of the data’s underlying structure.
\begin{figure}[t]
    \centering
    \begin{minipage}[t]{0.46\textwidth}
        \centering
        \includegraphics[width=\textwidth]{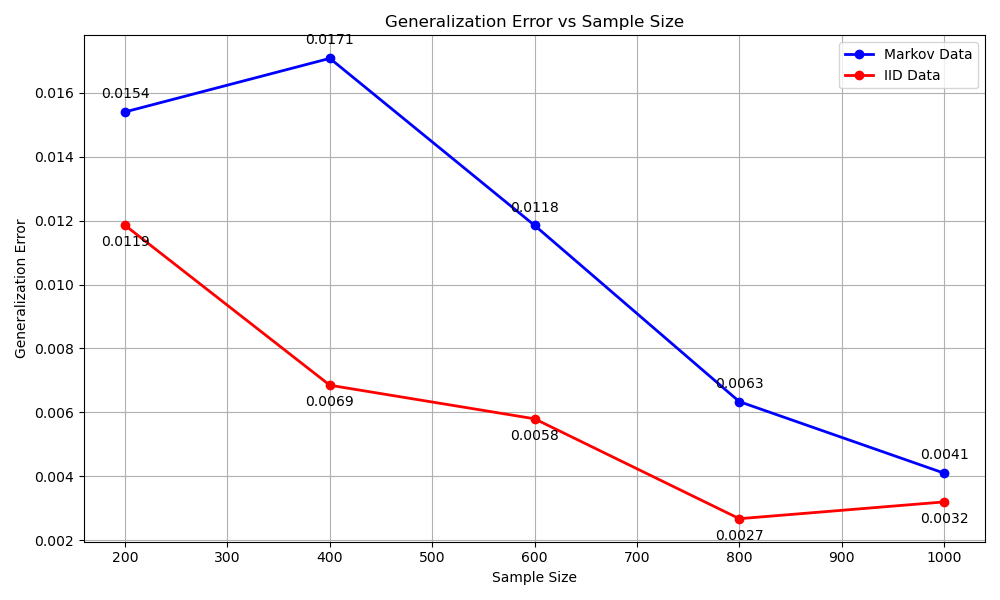} 
        \caption{Markov Property and Generalization}
        \label{fig:2}
    \end{minipage}
    \hfill 
    \begin{minipage}[t]{0.46\textwidth}
        \centering
        \includegraphics[width=1.08\textwidth]{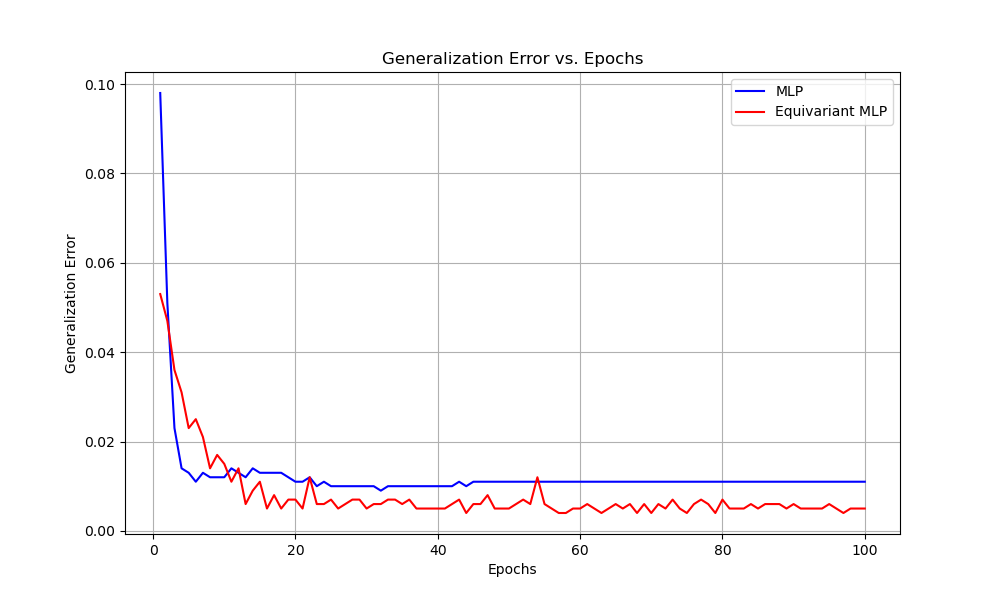}
        \caption{Equivariance and Generalization}
        \label{fig:3}
    \end{minipage}
\end{figure}
We begin by using a multi-layer perceptron (MLP) with the following architecture: a four-layer fully connected feedforward network. The network undergoes nonlinear transformations through ReLU activation functions. The network consists of an input layer, followed by a hidden layer with 32 neurons, another hidden layer with 16 neurons, and a final output layer. The output is passed through a Sigmoid activation function to produce a probability score. The model is trained using the Adam optimizer with $L_2$-regularization (weight decay). For the datasets, we use the two datasets mentioned earlier with $D=1$ to train the MLP. The results, shown in Fig.~\ref{fig:2}, demonstrate that the Markov property indeed reduces the network's generalization performance.

Next, we trained the $SO(2)$-equivariant steerable network (employs the equivariant layer from \cite{weiler2019general}) with the same depth and width as the MLP above on the Markov dataset with $D=1$. The results are shown in Fig.~\ref{fig:3}. It can be observed that for the Markov dataset, which inherently has certain symmetries, the equivariant network indeed leads to an increase in generalization error.
\begin{figure}[t]
    \centering
    \begin{minipage}[t]{0.46\textwidth}
        \centering
        \includegraphics[width=\textwidth]{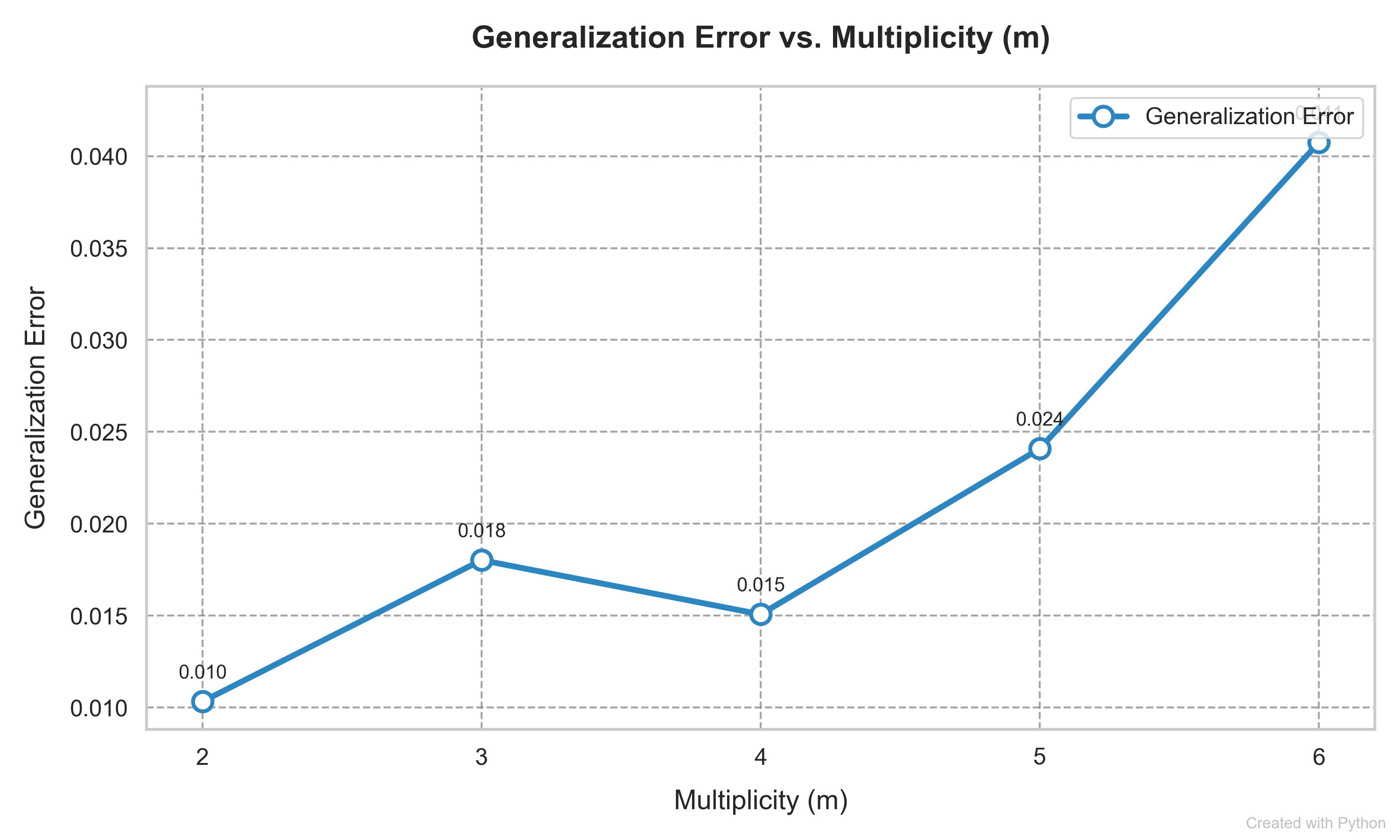} 
        \caption{The generalization error obtained from training equivariant networks on the Markov dataset on high-dimensional torus for different multiplicities $m$ .}
        \label{fig:4}
    \end{minipage}
    \hfill 
    \begin{minipage}[t]{0.46\textwidth}
        \centering
        \includegraphics[width=\textwidth]{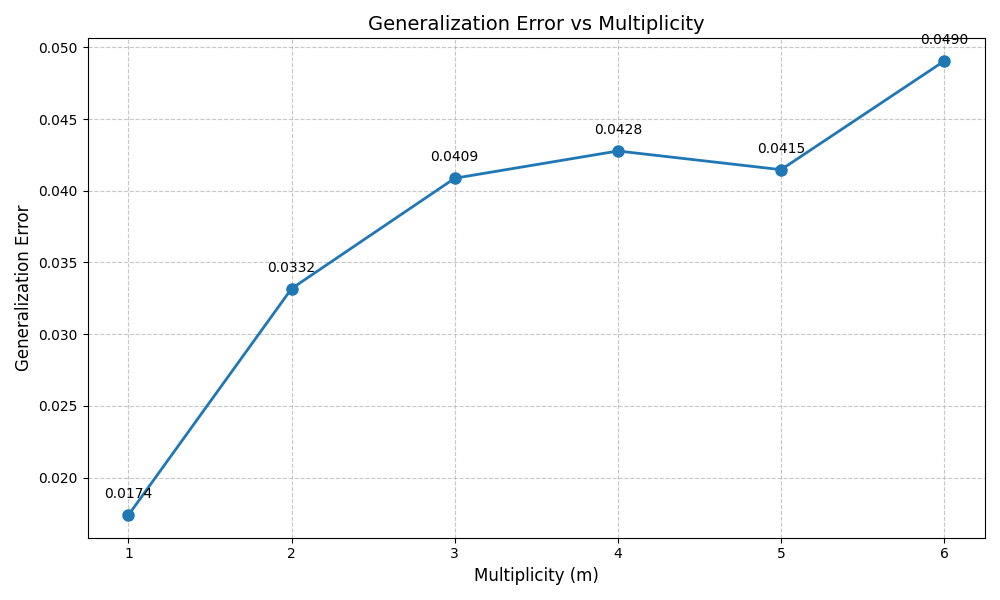}
        \caption{The generalization error obtained by training the Rotated MNIST dataset with equivariant networks at different multiplicities $m$ .}
        \label{fig:5}
    \end{minipage}
\end{figure}
  
Finally, we train both the generated Markov dataset on the high-dimensional torus ($D>1$) and the official rotated MNIST dataset, within the framework of the equivariant network. The results are shown in Fig.~\ref{fig:4} and Fig.~\ref{fig:5}. For each layer, we set the multiplicity $m_{l,\psi}$	of the irreducible representations to be equal to a constant $m$, and then vary $m$ to adjust the multiplicities. This allows us to observe the generalization error under different values of $m$. The results indicate that, when adjusting the multiplicities of the irreducible representations, the generalization error fluctuates to some extent. However, the overall trend suggests that smaller multiplicities generally perform better, whether for independent and identically distributed datasets or datasets with Markov properties. This is consistent with the theoretical results in Remark~\ref{remark}

\section{Conclusion}\label{con}
This paper has provided a unified framework for understanding the generalization behavior of neural networks trained on Markov datasets by adapting McDiarmid’s inequality to account for Markov dependencies and incorporating mixing time into the analysis. The resulting generalization bound highlights how the discrepancy between the initial and stationary distributions influences model performance. To address the structural constraints of equivariant networks, we leveraged the Peter-Weyl theorem and Schur’s lemma to characterize the block diagonal structure of weight matrices required for group symmetry. Using Maurey’s sparsification lemma, we quantified the reduced complexity of the hypothesis space, leading to tighter bounds on the empirical Rademacher complexity. These contributions culminate in a novel generalization bound that emphasizes the benefits of selecting non-isomorphic irreducible representations to enhance performance.

\bibliographystyle{siamplain}
\bibliography{references}

\newpage

\appendix
\section{Real Valued Representation}\label{R_representation}
In this section, we discuss the representation theory of compact groups on the real field $\mathbb{R}$.

\begin{dfn}(Representation).
	A representation of the  group G on the vector space $V$ is a  continuous homomorphism $\rho:G\to GL(V)$ which  associates to each element of $g\in G$ an element of general linear group $GL(V)$, such that the condition below is satisfied:
	$$\forall g,h\in G,\quad \rho(gh)=\rho(g)\rho(h).$$
\end{dfn}
In this context, $V$ mainly refers to Euclidean space $\mathbb{R}^n$.

Given two representations $\rho_1 : G \to GL(\mathbb{R}^{n_1})$ and $\rho_2 : G \to GL(\mathbb{R}^{n_2})$, their direct sum
$\rho_1\oplus\rho_2 : G \to GL(\mathbb{R}^{n_1+n_2})$ is a representation obtained by stacking the two representations as
follow:
\begin{equation*}
	\rho_1\oplus\rho_2(g)=\begin{bmatrix}
		\rho_1(g) & 0 \\
		0 & \rho_2(g) 
	\end{bmatrix}.
\end{equation*}
\begin{dfn}(Irreducible Representation).
For a representation $\rho: G \to GL(V )$, if there exists a basis transformation matrix $Q$, such that $\rho=Q^{-1}[\rho_1\oplus\rho_2]Q$, then $\rho$ is a reducible representation; otherwise, it is an irreducible representation(irreps).
\end{dfn}
\begin{dfn}(Equivalent Representation)
  For representation $\rho_1,\rho_2$, if there exists a change of basis matrix $Q$ such that $\rho_1=Q^{-1}\rho_2 Q$, then $\rho_1$ and $\rho_2$ are called equivalent representations; otherwise, they are non-equivalent representations.  
\end{dfn}
For a compact group $G$, we usually denote the set of all its non-equivalent irreps as $\widehat{G}$, indexed by a set of indices $I$.

\begin{lem}(Schur lemma for real-valued representation)\label{Schur_lemma}
	Let $\rho_1:G\to GL(\mathbb{R}^{\dim_{\rho_1}})$ and $\rho_2:G\to GL(\mathbb{R}^{\dim_{\rho_2}})$ 
	be two real valued irreps of $G$ respectively
	on vector spaces $V_1$ and $V_2$. Suppose that the linear transformation $W : V_1 \to V_2$ is equivariant, i.e. $W\rho_1(g)=\rho_2(g)W$. 
	Then:
	
	1. If $\rho_1$ and $\rho_2$ are non-isomorphic, then $W = 0$.
	
	2. If $V_1 = V_2$ and $\rho_1=\rho_2=\rho$, $\dim_{\rho_1}=\dim_{\rho_2}=d$, then 
	
	• $\rho$ is real type: $W=W_{\mathbb{R}}=\lambda \textbf{I}_{d}$ ($\textbf{I}_d$ is a identity matrix of dimension d).
	
	• $\rho$ is complex type: \begin{equation*}
		W=W_{\mathbb{C}}=    \left[
            \begin{array}{cc}
              a & -b \\
              b & a \\
            \end{array}
          \right]
     \otimes \textbf{I}_{d/2}.
 	\end{equation*}
	
	• $\rho$ is quaternionic type: \begin{equation*}
		W=W_{\mathbb{H}}=     \left[
            \begin{array}{cccc}
        a & -c & -b & -d \\
        c & a & d & -b \\
        b & -d & a & c \\
        d & b & -c & a \\
            \end{array}
          \right]\otimes \textbf{I}_{d/4}.
	\end{equation*}
\end{lem}
The columns of $ W $ are all orthogonal to each other and each contains only one copy of each of its $ c_\rho $ free parameters, where
\begin{equation*}
c_\rho =
\begin{cases}
1 & \text{real type} \\
2 & \text{complex type} \\
4 & \text{quaternionic type}.
\end{cases}
\end{equation*} 
For more details refer to (\cite{cesa2022program}, Appendix) and \cite{brocker2013representations}.

\section{Proof of  Theorem~\ref{thm 1}}\label{Appendix B}
Let $\mathcal{L}(Z)$ denote the set of real-valued signed measures on $ (Z, \mathcal{B}(Z)) $, where $\mathcal{B}(Z)$ is the set of all Borel subsets in $ Z $. Assume that $\nu \ll \pi$. Define
\begin{equation*}
\mathcal{L}_2 := \left\{ \nu \in \mathcal{L}(Z) : \left\| \frac{d\nu}{d\pi} \right\|_2 < \infty \right\},
\end{equation*}
where $\| \cdot \|_2$ is the standard $L_2$ norm in the Hilbert space of complex-valued measurable functions on $Z$.
\begin{lem}\label{lem2}\cite[Theorem 3.41]{rudolf2011explicit}
	Let $\{z_n\}_{n\in \mathbb{N}}$ be a stationary Markov chain 
 with an initial distribution $\nu\in \mathcal L_2$ and 
 a stationary distribution $ \pi $.  Let $h\in\mathcal{F}_\gamma$ and define
	\begin{equation*}
		S_{n,n_0}(h)=\frac{1}{n}\sum_{j=1}^{n}h(x_{j+n_0}) 
	\end{equation*}
	for all $n_0>0$. Assume that $h$ is uniformly bounded by $M$ and the Markov operator has an $L_2$-spectral gap, i.e. $1-\beta > 0$. Then, it holds that
	\begin{equation}\label{10}
		\mathbb{E}\left[\left(S_{n,n_0}(h)-\mathcal S[h(x)]\right)^2\right]\leq\frac{2M}{n(1-\beta)}+\frac{64M^2}{n^2(1-\beta)^2}\beta^{n_0}\left\|\frac{d\nu}{d\pi}-1\right\|_2,
	\end{equation}
\end{lem}
where $\beta$ is defined in (\ref{eqn_beta}).

\begin{lem}\label{lem A.4}(\cite{paulin2015concentration},Corollary 2.10,McDiarmid’s inequality for Markov chains).
	Let $S=(z_1,z_2,\ldots,z_n)$ be a uniformly ergodic Markov chain with mixing time $ t_{mix}(\epsilon) $(for $ 0 \leq \epsilon \leq 1 $). Let
 \begin{align}\label{eqn_tau_min}
     \tau_{\min} := \inf_{0 \leq \epsilon < 1} t_{mix}(\epsilon) \cdot \left( \frac{2 - \epsilon}{1 - \epsilon} \right)^2.
 \end{align}
	Suppose that $ h : Z \to \mathbb{R} $ satisfies 
	\begin{equation*} h(z) - h(z') \leq \sum_{i=1}^{n} c_i \mathbf{1}_{\{z_i = z_i'\}} \end{equation*}
	 for some $ c=(c_1,\ldots,c_n)\in \mathbb{R}_{+}^n$. Then for any $ t \geq 0 $,
  \begin{align}\label{eqn_mc}
	\emph{Pr}\big(h(S) - \mathbb{E}[h(S)] \geq t\big) \leq  \exp \left( - \frac{2t^2}{ \|c\|_2^2 \tau_{\min}} \right).
 \end{align}
\end{lem}

\begin{lem}\label{lem1}(\cite{truong2022generalization})
	Let $\{z_i\}^n_{i=1}$ be an arbitrary Markov chain on $Z$, and let $\{z_i'\}^n_{i=1}$ be a
	independent copy (replica) of $\{z_i\}^n_{i=1}$. Denote by $S = (z_1, z_2, \ldots, z_n)$, $S' = (z_1', z_2', \ldots , z_n')$,
	and $\mathcal{F_\gamma}$ a class of  functions from $Z\to\mathbb{R}$. Let $\bm{\epsilon}=(\epsilon_1,\ldots,\epsilon_n) $ be a vector
	of i.i.d. Rademacher’s random variables. Then, the following holds:
	\begin{equation}\label{9}
		E_\epsilon\left[E_{S,S'}\left[\sup_{h\in\mathcal{F_\gamma}}\sum_{i=1}^{n}\epsilon_i(h(z_i)-h(z_i'))\right]\right]=E_{S,S'}\left[\sup_{h\in\mathcal{F_\gamma}}\sum_{i=1}^{n}(h(z_i)-h(z_i'))\right].
	\end{equation}
\end{lem}

\begin{lem}\label{lemB.1}(\cite{bartlett2017spectrally})
For any function $ f : \mathbb{R}^d \to \mathbb{R}^k $ and every $ \gamma > 0 $,
\begin{equation*}
\Pr[{\arg \max}_i f(x)_i = y] \leq \Pr[\mathcal M(f(x), y) \leq 0] \leq R_{\gamma}(f),
\end{equation*}
where the \text{arg max} follows any deterministic tie-breaking strategy.
\end{lem}
Based on the above lemmas, we provide the proof of Theorem~\ref{thm 1} as follows.
\begin{proof}
	For any Markov dataset $S=\{z_1,\ldots,z_n\}$, where $z_i=(x_i,y_i)$ for any $i\in\{1,2,\ldots,n\}$, and any $h\in\mathcal{F}_\gamma$, we have $\mathcal{R}_\gamma(f)=\mathcal S(h(z))=\int h(z)\pi(z)dz$ and $\widehat{\mathcal{R}}_\gamma(f)=\frac{1}{n}\sum_{i=1}^nh(z_i)$. Since the empirical risk  $\widehat{\mathcal{R}}_\gamma(f)$ depends on the data set $S$, we define $\widehat{\mathbb{E}}_S(h(z))=\widehat{\mathcal{R}}_\gamma(f)$.
   
 The proof consists of applying
	McDiarmid’s inequality  to the following function $\Phi$ defined for any Markov data $S$ by $$\Phi(S)=\sup_{h\in\mathcal{F}_\gamma}\mathcal S(h(z))-\widehat{\mathbb{E}}_S(h(z)).$$
	Let $S$ and $S'$ be two samples differing by one point, say $z_m$ in $S$ and $z_m'$ in $S'$. Then, since $\sup{A}-\sup{B}\leq \sup(A-B)$, we have
	\begin{align}
		\Phi(S)-\Phi(S')
		\leq \sup_{h\in\mathcal{F}_\gamma}(\widehat{\mathbb{E}}_{S'}(h(z))-\widehat{\mathbb{E}}_S(h(z)))
  =\sup_{h\in\mathcal{F}_\gamma}\frac{h(z_m')-h(z_m)}{n}
		\leq\frac{1}{n},
	\end{align}
 the last inequality is due to $h(z)\in[0,1].$
	Then by Lemma $\ref{lem A.4}$, for any $\delta>0$, with probability at least $1-\delta/2$, the following holds
	\begin{equation}\label{14}
		\Phi(S)\leq \mathbb{E}[\Phi(S)]+\sqrt{\frac{\tau_{min}\ln(2/\delta)}{2n}}.
	\end{equation}
	Next, we bound the expectation of the right-hand side as follows:
	\begin{align}\label{15}
		\nonumber \mathbb{E}[\Phi(S)]
		& =\mathbb{E}\left[\sup_{h\in\mathcal{F}_\gamma}\left(\mathcal S(h(z))-\widehat{\mathbb{E}}_S(h)\right)\right] \\
		& \nonumber=\mathbb{E}\left[\sup_{h\in\mathcal{F}_\gamma}\left(\mathcal S(h(z))-\mathbb{E}(\widehat{\mathbb{E}}_{S'}(h))+\mathbb{E}(\widehat{\mathbb{E}}_{S'}(h))-\widehat{\mathbb{E}}_S(h)\right)\right] \\
		& \leq \mathbb{E}\left[\sup_{h\in\mathcal{F}_\gamma}\left(\mathcal S(h(z))-\mathbb{E}(\widehat{\mathbb{E}}_{S'}(h))\right)\right] +\mathbb{E}\left[\sup_{h\in\mathcal{F}_\gamma}(\mathbb{E}(\widehat{\mathbb{E}}_{S'}(h))-\widehat{\mathbb{E}}_S(h))\right].
	\end{align}
  In the case of i.i.d. data, $\mathcal S(h(z))=\mathbb{E}(\widehat{\mathbb{E}}_{S'}(h))$, 
but for Markov datasets, we have 
	\begin{align*}
		\mathcal S(h(z))-\mathbb{E}(\widehat{\mathbb{E}}_{S'}(h))
		&=\mathbb{E}[\mathcal S(h(z))-\frac{1}{n}\sum_{i=1}^{n}h(z_i')]  \\
		&= \mathbb{E}[\mathcal S(h(z))-S_{n,0}(h)]\\
		&\leq \sqrt{\mathbb{E}[(\mathcal S(h(z))-S_{n,0}(h))^2]}\quad(\text{Jensen's inequality})\\
		&\leq \sqrt{\frac{2}{n(1-\beta)}+\frac{64}{n^2(1-\beta)^2}\left\|\frac{d\nu}{d\pi}-1\right\|_2}=C_n,\quad(\text{Lemma~\ref{lem2}})
	\end{align*}
	which indicates that   $\mathbb{E}[\sup_{h\in\mathcal{F}_\gamma}(\mathcal S(h(z))-\mathbb{E}(\widehat{\mathbb{E}}_{S'}(h))]\leq C_n$.
	
	On the other hand, 
	\begin{align*}
		&\mathbb{E}[\sup_{h\in\mathcal{F}_\gamma}(\mathbb{E}(\widehat{\mathbb{E}}_{S'}(h))-\widehat{\mathbb{E}}_S(h))]\\
		&=\mathbb{E}[\sup_{h\in\mathcal{F}_\gamma}(\mathbb{E}_{S'}(\frac{1}{n}\sum_{i=1}^{n}h(z_i')-\frac{1}{n}\sum_{i=1}^{n}h(z_i))] \\
		&\leq \mathbb{E}[\mathbb{E}_{S'}[\sup_{h\in\mathcal{F}_\gamma}(\frac{1}{n}\sum_{i=1}^{n}h(z_i')-\frac{1}{n}\sum_{i=1}^{n}h(z_i))]]&(\text{sub-additivity
of sup.)}  \\
		&=\mathbb{E}_{S,S'}[\sup_{h\in\mathcal{F}_\gamma}(\frac{1}{n}\sum_{i=1}^{n}(h(z_i')-h(z_i)))] \\
		&=\mathbb{E}_{\epsilon,S,S'}[\sup_{h\in\mathcal{F}_\gamma}\epsilon_i(\frac{1}{n}\sum_{i=1}^{n}(h(z_i')-h(z_i)))] &(\text{Lemma~\ref{lem1}}) \\
		&\leq \mathbb{E}[\sup_{h\in\mathcal{F}_\gamma}\frac{1}{n}\sum_{i=1}^{n}\epsilon_ih(z_i')]+ \mathbb{E}[\sup_{h\in\mathcal{F}_\gamma}\frac{1}{n}\sum_{i=1}^{n}\epsilon_ih(z_i)] \\
		&\leq 2\mathbb{E}[\mathfrak{R}_S(\mathcal{F}_\gamma)] .
	\end{align*}
	Then, with probability $1-\delta$ the following holds:
	\begin{equation}
		R_\gamma(f)\leq\hat{R}_\gamma(f)+ 2\mathbb{E}[\mathfrak{R}_S(\mathcal{F}_\gamma)] +\sqrt{\frac{\tau_{min}\ln(2/\delta)}{2n}}+C_n.
	\end{equation}
	To derive a bound in terms of $ \mathfrak{R}_S(\mathcal{F}_\gamma) $, we observe that, by the definition of empirical Rademacher complexity, changing one point in $ S $ changes $ \mathfrak{R}_S(\mathcal{F}_\gamma) $ by at most $ \frac{1}{n} $. Then,  using Lemma~\ref{lem A.4} again, with probability $ 1 - \frac{\delta}{2} $ the following holds:
	\begin{equation*} E[\mathfrak{R}_S(\mathcal{F}_\gamma)]\leq \mathfrak{R}_S(\mathcal{F}_\gamma) + \sqrt{\frac{\tau_{min}\ln(2/\delta)}{2n}}.
	\end{equation*}
 Finally, we obtain the final result using Lemma~\ref{lemB.1}.
\end{proof}
\section{Proof of Lemma~\ref{linear_covering}}\label{Appendix C}

First, let us recall the Maurey sparsification lemma.

\begin{lem}\label{lem_sparsification}(Maurey, cf. \cite{pisier1981remarques})
	Fixed a Hilbert space $\mathcal H$ with norm $\|\cdot\|$. Let $U\in\mathcal H$ be given with representation $U =\sum_{i=1}^{d}\alpha_iV_i$ where $V_i\in\mathcal H$ and $\alpha_i \geq 0$. Then for any
	positive integer $K$, there exists a choice of nonnegative integers $(k_1, \ldots , k_d)$, $\sum_{i=1}^dk_i=K$, such that
	\begin{equation*}
		\left\|U-\frac{\|\alpha\|_1}{K}\sum_{i=1}^{d}k_iV_i\right\|^2\leq\frac{\|\alpha\|_1}{K}\sum_{i=1}^{d}\alpha_i\|V_i\|^2\leq\frac{\|\alpha\|^2_1}{K}\max_i\|V_i\|^2.
	\end{equation*}
\end{lem}

Based on the Maurey sparsification lemma, we begin the proof of Lemma~\ref{linear_covering}.
\begin{proof}

First, by the conclusion of Lemma~\ref{lem_eq}, finding a covering for $\mathcal{H}'_l$ is equivalent to finding one for $\widehat{\mathcal{H}}_l$. Therefore, it suffices to compute the covering number for $\widehat{\mathcal{H}}_l$. This process can be broken down into three main steps.

\textbf{Step 1(Construct the $\epsilon$-cover $\mathcal C$):}

Let the matrix $\widehat X^{l-1}\in\mathbb R^{d_{l-1}\times n}$ be given, and construct the matrix $Y\in\mathbb R^{d_{l-1}\times n}$ by rescaling each row of $\widehat X^{l-1}$ to have unit 2-norm: $Y_{ij}:=\widehat X^{l-1}_{ij}/a_i$, where $a_i=\|\widehat X^{l-1}_{i:}\|_2$ is the 2-norm of the $i$-th row of $\widehat X^{l-1}$. Set $N := 2\sum_{k}c_{\psi_k}m_{l,\psi_k}m_{l-1,\psi_k}\dim_{\psi_k}$, $K:=\lceil \max_k c_{\psi_k}m_{l,\psi_k} \| W_l \|_F^2 \| X_{l-1} \|_F^2/\epsilon^2\rceil$ and $\bar a:=\max_{k}\sqrt{c_{\psi_k}m_{l,\psi_k}}\| W_l\|_F\| X^{l-1}\|_{F}$. Next, define
	\begin{align}
    \{V_1, \cdots, V_N\} := & \Big\{ m \bm{e}_{a,j,k} \bm{e}_{b,i,k}^T Y : \, m \in \{-1, +1\}, \, a \in \{1, \dots, \dim_{\psi_k}\}, \notag \\
    & \quad b \in \{ a, \frac{\dim_{\psi_k}}{c_{\psi_k}} + a, \frac{\dim_{\psi_k}}{c_{\psi_k}} + 2a, \dots, \frac{\dim_{\psi_k}}{c_{\psi_k}} + (c_{\psi_k} - 1)a \bmod \dim_{\psi_k} \}, \notag \\
    & \quad j \in \{1, \dots, m_{l,\psi_k}\}, \, i \in \{1, \dots, m_{l-1,\psi_k}\}, \, k \in I \Big\}. \label{eq:V_set}
\end{align}
in which $\bm{e}_{a,j,k} ,\bm{e}_{b,i,k}$ represent the unit vectors of the 
$(\psi_k, j,i)$ block, with their $a$-th and $b$-th components set to 1, respectively.
\begin{equation}\label{eq_cover}
    \mathcal{C} := \left\{ \frac{\bar{a}}{K} \sum_{i=1}^{N} k_i V_i : k_i \geq 0, \sum_{i=1}^{N} k_i = K \right\} = 
    \left\{ \frac{\bar{a}}{K} \sum_{i=1}^{N} V_{i_j} : (i_1, \cdots, i_k) \in [N]^K \right\}
\end{equation}
where the $k_i$'s are integers.
 
\textbf{Step 2 (Demonstrate that $\mathcal C$ is the desired cover):}

 Now with the definition of $V_i$ and $Y$ implies
	\begin{equation*}
	\max_i\|V_i\|_F\leq\max_j\|\bm{e}_j^T Y\|_2=\max_i\frac{\|\bm{e}_j^T X\|_2}{\|\bm{e}_j^T X\|_2}=1.
	\end{equation*}

We will construct a cover element within  $\mathcal C$ using the following technique: the basic Maurey lemma is applied to  non-$l_1$ norm balls simply by rescaling.

$\bullet$  Define $\alpha\in\mathbb R^{d_l\times d_{l-1}}$ to be a "rescaling matrix" where every element of column $i$ is equal to $a_i$; the purpose of $\alpha$ is to annul the rescaling of $\widehat X^{l-1}$ introduced by $Y_{l-1}$, meaning $\widehat W_l\widehat X^{l-1}=(\alpha\odot \widehat W_l)Y_{l-1}$ where "$\odot$" denotes element-wise product. 

$\bullet$ Define $B:=\alpha\odot \widehat W_l$, whereby using the fact that $\widehat W_l$ is a block diagonal matrix, for clarity, using (\ref{W_decomposition}) as a special case, we obtain:
\begin{equation*}
	\resizebox{\textwidth}{!}{$
		B = \begin{bmatrix}
			\left[\widehat W_l(\psi_1,j,i)\odot\alpha(\psi_1,j,i)\right] & 0 & \cdots & 0 \\
			0 & \left[\widehat W_l(\psi_2,j,i)\odot\alpha(\psi_2,j,i)\right] & 0 & 0 \\
			\vdots & \cdots & \cdots & \vdots \\
			0 & \cdots & 0 & \left[\widehat W_l(\psi_C,j,i)\odot\alpha(\psi_C,j,i)\right]
		\end{bmatrix}
		$}
\end{equation*}
	Here, we partition matrix $\alpha$, dividing it into blocks $\alpha(\psi_k,j,i)$ of size $\text{dim}_{\psi_k}\times \text{dim}_{\psi_k}$ each. Then
$\|B\|_1=\sum_{i,j,k}\|\widehat W_l(\psi_k,j, i)\odot\alpha(\psi_k,j, i)\|_1$, we need to bound $\|\widehat W_l(\psi_k,j, i)\odot\alpha(\psi_k,j, i)\|_1$, Since $\psi_k$  has three types, we will discuss it in three separate cases. First, if $\psi_k$ is real type, i.e. $c_{\psi_k}=1$, then 
$\widehat W_l(\psi_k,j, i)=\lambda\textbf{I}_{\text{dim}_{\psi_k}}$,
\begin{align*}
	\|\widehat W_l(\psi_k,j, i)\odot\alpha(\psi_k,j, i)\|_1&=\sum_{m=1}^{\dim_{\psi_k}}|\lambda\alpha_{mm}(\psi_k,j, i)| \\
    &\leq \sqrt{\sum_{m=1}^{\dim_{\psi_k}}\lambda^2}\sqrt{\sum_{m=1}^{\dim_{\psi_k}}\alpha_{mm}^2(\psi_k,j, i)}\\
 &=\|\widehat W_l(\psi_k,j, i)\|_F\sqrt{\sum_{m=1}^{\dim_{\psi_k}}\alpha_{mm}^2(\psi_k,j, i)}.
\end{align*}
Second, if $\psi_k$ is complex type, i.e. $c_{\psi_k}=2$, then $\widehat W_l(\psi_k,j, i)=\left[
            \begin{array}{cc}
              a & -b \\
              b & a \\
            \end{array}
          \right]
     \otimes \textbf{I}_{\text{dim}_{\psi_k}/2},$
\begin{align*}
	\|\widehat W_l(\psi_k,j, i)\odot\alpha(\psi_k,j, i)\|_1&=\sum_{m=1}^{\dim_{\psi_k}}(|a\alpha_{mm}(\psi_k,j, i)|+|b\alpha_{mm}(\psi_k,j, i)|)\\
 &\leq \sqrt{\sum_{m=1}^{\dim_{\psi_k}}(a^2+b^2)}\sqrt{2\sum_{m=1}^{\dim_{\psi_k}}\alpha_{mm}^2(\psi_k,j, i)}\\
 &=\sqrt{2}\|\widehat W_l(\psi_k,j, i)\|_F\sqrt{\sum_{m=1}^{\dim_{\psi_k}}\alpha_{mm}^2(\psi_k,j, i)}.
\end{align*}
Third, if $\psi_k$ is  quaternionic type, i.e. $c_{\psi_k}=4$, then 
$$\widehat W_l(\psi_k,j, i)=\left[
            \begin{array}{cccc}
        a & -c & -b & -d \\
        c & a & d & -b \\
        b & -d & a & c \\
        d & b & -c & a \\
            \end{array}
          \right] \otimes \textbf{I}_{\text{dim}_{\psi_k}/4},$$
\begin{align*}
	\|\widehat W_l(\psi_k,j, i)\odot\alpha(\psi_k,j, i)\|_1&=\sum_{m=1}^{\dim_{\psi_k}}(|a|+|b|+|c|+|d|)\alpha_{mm}(\psi_k,j, i)\\
 &\leq \sqrt{\sum_{m=1}^{\dim_{\psi_k}}(a^2+b^2+c^2+d^2)}\sqrt{4\sum_{m=1}^{\dim_{\psi_k}}\alpha_{mm}^2(\psi_k,j, i)}\\
 &=\sqrt{4}\|\widehat W_l(\psi_k,j, i)\|_F\sqrt{\sum_{m=1}^{\dim_{\psi_k}}\alpha_{mm}^2(\psi_k,j, i)}.
\end{align*}
Based on the above three results, we can conclude:
 
	\begin{align*}
    \|B\|_1 &\leq \sum_{i,j,k} \sqrt{c_{\psi_k}} \|\widehat{W}_l(\psi_k,j,i)\|_F 
\sqrt{\sum_{m=1}^{\dim_{\psi_k}} \alpha_{mm}^2(\psi_k,j,i)} \\
&\leq \sqrt{\sum_{i,j,k} \|\widehat{W}_l(\psi_k,j,i)\|_F^2} 
\sqrt{\sum_{i,j,k}c_{\psi_k} \sum_{m=1}^{\dim_{\psi_k}} \alpha_{mm}^2(\psi_k,j,i)} \quad (\text{Cauchy-Schwarz inequality})\\
 & \leq \sqrt{\sum_{i,j,k} \|\widehat{W}_l(\psi_k,j,i)\|_F^2} \sqrt{\max_k c_{\psi_k} m_{l,\psi_k} \sum_k \sum_{m=1}^{\dim_{\psi_k}} \sum_{i=1}^{m_{l-1,\psi_k}} \alpha_{mm}^2(\psi_k,j,i)} \\
    &=  \sqrt{\sum_{i,j,k} \|\widehat{W}_l(\psi_k,j,i)\|_F^2} \sqrt{\max_k c_{\psi_k}m_{l,\psi_k} \sum_{i=1}^{d_{l-1}} a_i^2} \\
    &= \max_k  \sqrt{\sum_{i,j,k} \|\widehat{W}_l(\psi_k,j,i)\|_F^2} \sqrt{\max_k c_{\psi_k} m_{l,\psi_k} \| X^{l-1} \|_F} \\
    &= \max_k \sqrt{c_{\psi_k}m_{l,\psi_k}} \|\widehat{W}_l\|_F \| \widehat{X}^{l-1} \|_F \\
    &= \max_k \sqrt{c_{\psi_k}m_{l,\psi_k}} \| W_l \|_F \| X_{l-1} \|_F = \bar{a}.
\end{align*}

	Consequently, 
	\begin{equation*}
		\widehat W_l\widehat X^{l-1}=\sum_{i,j,B_{ij}\neq0}B_{ij}\bm{e}_i\bm{e}_j^TY=\|B\|_1\sum_{i}\frac{B_{ij}}{\|B\|_1}\bm{e}_i\bm{e}_j^TY\in\bar{a}\cdot \text{conv}(\{V_1,\cdots,V_N\}),
	\end{equation*}
	where conv($\{V_1,\cdots,V_N\}$) is the convex hull of $\{V_1,\cdots,V_N\}$.
	
$\bullet$	Combining the preceding constructions with Lemma~\ref{lem_sparsification}, there exist nonnegative integers $(k_1,\cdots,k_N)$ with $\sum_ik_i=K$ with 
	\begin{align*}
		\left\|\widehat W_l\widehat X^{l-1}-\frac{\bar{a}}{K}\sum_{i=1}^{N}k_iV_i\right\|_F^2&= \left\|BY_{l-1}-\frac{\bar{a}}{K}\sum_{i=1}^{N}k_iV_i\right\|_F^2\\
		&\leq\frac{\bar{a}^2}{K}\max_i\|V_i\|_F\\
		&\leq\frac{\max_k c_{\psi_k}m_{l,\psi_k}\| W_l \|_F^2\| X_{l-1}\|_{F}^2}{K}\leq\epsilon^2.
	\end{align*} 
The desired cover element is thus $\frac{\bar{a}}{K}\sum_{i=1}^{N}k_iV_i\in\mathcal{C}$.
 
\textbf{Step 3 (Calculate  the upper bound of $|\mathcal C|$):}
From Equation (\ref{eq:V_set}), we calculate $N=2\sum_{k}c_{\psi_k}m_{l,\psi_k}m_{l-1,\psi_k}\dim_{\psi_k}=2D_l$. Subsequently, using Equation (\ref{eq_cover}), it follows that $|\mathcal C|\leq N^K$. Thus, by applying Lemma 5, we conclude:
	\begin{align*}
		\ln\mathcal N(W_lX_{l-1},\|\cdot\|_F,\epsilon)&=\ln\mathcal N(\widehat W_l\widehat X^{l-1},\|\cdot\|_F,\epsilon)\\&\leq\ln|\mathcal{C}|\\
        &\leq \left\lceil\frac{\max_k c_{\psi_k}m_{l,\psi_k}\| W_l\|_F^2\| X_{l-1}\|_{F}^2}{\epsilon^2}\right\rceil\ln(2D_l).
	\end{align*}
\end{proof}
\section{The proof of Lemma~\ref{lemma7}}

The proof of this lemma is structured into two main steps. First, we employ induction over the layers to establish the cover of the entire network. Second, we determine the resolution $\epsilon_l$ for each layer based on $\epsilon$. Using Lemma~\ref{linear_covering}, we compute the logarithm of the covering number for each layer and sum these values to obtain the logarithm of the covering number for the entire network.

\textbf{Step 1: Construct the $\epsilon$-cover}

Let $\mathcal B_l=\{W_l,W_l\rho_{l-1}=\rho_{l}W_l,\|W_l\|_2\leq s_l\},l\in[L]$.
Inductively construct covers $\mathcal C_1,\dots,\mathcal C_L$ as follows.

$\bullet$ Choose an $\epsilon_1$-cover $\mathcal C_1$ of $\{W_1X,W_i\in \mathcal B_1\}$, thus
\begin{equation*}
    \mathcal N(\{W_1X,W_1\in \mathcal B_1\},\|\cdot\|_F,\epsilon_1)\leq |\mathcal C_1|:=N_1
\end{equation*}

$\bullet$ For every element $X'_l\in \mathcal{C}_l$, construct an $\epsilon_{l+1}$-cover $\mathcal C_{l+1}(X'_l)$ of $\{W_{l+1}\sigma_{l}X'_{l}$, $W_{l+1}\in\mathcal B_{l+1}\}$, then
\begin{equation*}
   \mathcal N(\{W_{l+1}\sigma_{l}X'_{l},W_{l+1}\in\mathcal B_{l+1}\},\|\cdot\|_F,\epsilon_{l+1}) \leq |\mathcal C_{l+1}(X'_l)|:=N_{l+1}
\end{equation*}
Lastly, form the cover 
$$\mathcal C_{l+1}=\bigcup\limits_{X'_l\in \mathcal{C}_l} \mathcal C_{l+1}(X'_l),$$
whose cardinality satisfies
$$|\mathcal C_{l+1}|=|\mathcal C_{l}|N_{l+1}=\prod\limits_{i=1}^{l+1} N_i.$$

$\bullet$ Define $\mathcal C=\{\sigma_LX_L',X_L'\in \mathcal C_L\}$, by construction,
$$|\mathcal C|=|\mathcal C_L|=|\mathcal C_{L-1}|N_{L}=\prod\limits_{l=1}^{L} N_l.$$
 We will show that $\mathcal C$ serves as an $\epsilon$-cover of $\mathcal H_L$, where 
$\epsilon:= \sum_{j=1}^{L} \epsilon_{j}c_j \prod\limits_{l=j+1}^{L}c_{l}s_l .$ Denote $X_l=W_l\sigma_{l-1}W_{l-1}\dots,\sigma_1W_1X$, $l\in[L]$. For any $h\in \mathcal{H}_L$, there exists $h'\in\mathcal C$ such that:
\begin{align*}
    \|h-h'\|_F=&\|\sigma_LX_L-\sigma_LX'_L\|_F\\
    &\leq c_L\|X_L-X'_L\|_F\\
    &\leq c_L(\|W_L\sigma_{L-1}X_{L-1}-W_L\sigma_{L-1}X'_{L-1}\|_F+\|W_L\sigma_{L-1}X'_{L-1}-X'_L\|_F)\\
    &\leq c_Ls_L\|\sigma_{L-1}X_{L-1}-\sigma_{L-1}X'_{L-1}\|_F+c_L\epsilon_L\\
    &\vdots\\
    &\leq \sum_{j=1}^{L} \epsilon_{j}c_j \prod\limits_{l=j+1}^{L}c_{l}s_l .
\end{align*}

\textbf{Step 2: Calculate the covering number}

	The per-layer cover resolutions $(\epsilon_1,\ldots,\epsilon_L)$ set
	according to
	\begin{equation*}
		\epsilon_l:=\frac{\alpha_l\epsilon}{c_l\prod_{j>l}c_js_j}
	\end{equation*}
	where $\alpha_l:=\frac{1}{\bar\alpha}\left(\frac{\max_k \sqrt{c_{\psi_k}m_{l,\psi_k}} \| W_l \|_F}{s_l}\right)^{2/3}$ and $\bar\alpha:=\sum_{l=1}^{L}\left(\frac{\max_k \sqrt{c_{\psi_k}m_{l,\psi_k}} \| W_l \|_F}{s_l}\right)^{2/3}$.
	By this choice, it follows that the final cover resolution $\epsilon$ provided by Step 1 satisfies
	\begin{equation*}
     \sum_{j=1}^{L}\epsilon_jc_j\prod_{l=j+1}^{L}c_ls_l=\sum_{j=1}^{L}\alpha_j\epsilon=\epsilon.
	\end{equation*}
	
	Within the rest of the proof, a pivotal strategy involves utilizing the covering number estimates furnished by Lemma~\ref{linear_covering}.  
	\begin{align}
    & \ln\mathcal{N}(\mathcal{H}_L, \|\cdot\|_F, \epsilon)=\ln|\mathcal C|=\sum_{l=1}^{L}\ln|\mathcal C_l| \nonumber \\ 
    & \leq \sum_{l=1}^{L} 
    \frac{b_l^2 \|\sigma_{l-1}X_{l-1}\|_F^2}{\epsilon_l^2} \ln(2D_l) 
    \quad (\text{Lemma~\ref{linear_covering}}),\label{22}
    \end{align}
where $b_l=\max_k \sqrt{c_{\psi_k}m_{l,\psi_k}} \| W_l \|_F$,
	$D_l=\sum_{k}c_{\psi_k} m_{l,\psi_k}m_{l-1,\psi_k}\text{dim}\psi_k$.
	
 To simplify
	this expression, note for any $(W_1,\cdots,W_{l-1})$ that
	\begin{align}
	\|\sigma_{l-1}X_{l-1}\|_F& = \|\sigma_{l-1}X_{l-1} - \sigma_{l-1}(0)\|_F \nonumber \\
	& \leq c_{l-1}\|X_{l-1} - 0\|_F \nonumber \\
	& \leq c_{l-1}\|W_{l-1}\|_2 \|\sigma_{l-2}X_{l-2}\|_F \nonumber \\
	& \vdots \nonumber \\
	& \leq \|X\|_F \prod_{j=1}^{l-1} c_j s_j.\label{23}
\end{align}

	Combining (\ref{22}) and (\ref{23}) and subsequently plugging in the chosen value for $\epsilon_i$, we derive the following result:
	\begin{align*}
		\ln\mathcal N(\mathcal H_L,\|\cdot\|_F,\epsilon) & \leq\sum_{l=1}^{L}  \frac{b_l^2\|X\|_F^2\prod_{j=1}^{l-1}c_j^2s_j^2}{\epsilon_l^2}
		\ln(2D_l)\\
		& =\frac{\|X\|_F^2\ln(2\max_lD_l)\prod_{j=1}^{L}c_j^2s_j^2}{\epsilon^2}\sum_{l=1}^{L}
		\frac{\max_k c_{\psi_k}m_{l,\psi_k}\| W_l\|_F^2\| X_{l-1}\|_{F}^2}{\alpha_l^2s_l^2}\\
		& =\frac{\|X\|_F^2\ln(2\max_lD_l)\prod_{j=1}^{L}c_j^2s_j^2}{\epsilon^2}\left(\bar{\alpha}^3\right).
	\end{align*}

\section{Proof of Theorem~\ref{theorem_gb}}\label{proof_of_thm8}

\begin{lem}(Dudley Entropy Integral)\label{lem_Dudley}
	Let $\mathcal{F}$ be a real-valued function class taking values in $[0, 1]$, and assume that $0\in\mathcal{F}$. Then
	\begin{equation}\label{19}
		\mathfrak{R}_S(\mathcal F)\leq\inf_{\alpha>0}\left(\frac{4\alpha}{\sqrt{n}}+\frac{12}{n}\int_{\alpha}^{\sqrt{n}}
		\sqrt{\log\mathcal{N}(\mathcal{F}_{|S},\|\cdot\|_F),\epsilon)}d\epsilon\right).
	\end{equation}
\end{lem}
In the following, we begin to show the proof of Theorem~\ref{theorem_gb}.

 Consider the network class $\mathcal F_\gamma$ obtained by appending the ramp loss $l_\gamma$ and the margin operator $\mathcal M$ to the output of the given network class:
	\begin{equation*}
		\mathcal{F}_\gamma:=\{(x,y)\to l_\gamma(\mathcal{M}(f(x),y):f\in\mathcal{F}\}.
	\end{equation*}
	Since the function $(x, y) \to l_\gamma(\mathcal M(x, y))$ is $2/\gamma$-Lipschitz with respect to $\|\cdot\|_2$ under the definition of $\gamma$, the function class $\mathcal F_\gamma$ still falls within the setting of Lemma~\ref{lemma7} and thereby yields
	\begin{equation*}
		\ln\mathcal N(\mathcal H_S,\|\cdot\|_2,\epsilon) \nonumber\leq\frac{4\|X\|_F^2\ln(2D)
		\prod_{j=1}^{L}c_j^2s_j^2}{\gamma^2\epsilon^2}\left(\bar{\alpha}^3\right):=\frac{R}{\epsilon^2}.
	\end{equation*}
	
	The Dudley entropy integral bound on the Rademacher complexity from Lemma~\ref{lem_Dudley} yields:
	\begin{equation*}
		\mathfrak{R}_S(\mathcal F_\gamma)\leq \inf_{\alpha>0}\left(\frac{4\alpha}{\sqrt{n}}+\frac{12}{n}\int_{\alpha}^{\sqrt{n}}\sqrt{\frac{R}{\epsilon^2}}d\epsilon\right)
		=\inf_{\alpha>0}\left(\frac{4\alpha}{\sqrt{n}}+\ln(\sqrt{n}/\alpha)\frac{12\sqrt{R}}{n}\right).
	\end{equation*}
	The inf is uniquely attained at point $\alpha:=3\sqrt{R/n}$; for simplicity, we can choose $\alpha=1/\sqrt{n}$, and then substitute the obtained Rademacher complexity estimate into Theorem~\ref{thm 1}.

\end{document}

%% file: ex_shared.tex
\usepackage{amsmath}
\usepackage{amsfonts}
\usepackage{amssymb}
\usepackage{bm}
\usepackage{graphicx}
\usepackage{epstopdf}
\usepackage{algorithmic}
\usepackage{blindtext}
\usepackage{tikz}
\usepackage{pifont}
\usepackage{color}

\usepackage[numbers]{natbib}
\usepackage{lipsum}
\usepackage{dsfont}
\usepackage{arydshln}
\usepackage{extarrows}
\usepackage{circledsteps}
\usepackage{cleveref}
\newtheorem{thm}{\indent Theorem}

\newtheorem{lem}[thm]{\indent Lemma}

\newtheorem{dfn}{{\indent\bf Definition}} 

\ifpdf
  \DeclareGraphicsExtensions{.eps,.pdf,.png,.jpg}
\else
  \DeclareGraphicsExtensions{.eps}
\fi

\newsiamremark{remark}{Remark}
\newsiamremark{hypothesis}{Hypothesis}
\crefname{hypothesis}{Hypothesis}{Hypotheses}
\newsiamthm{claim}{Claim}
\newsiamremark{fact}{Fact}
\crefname{fact}{Fact}{Facts}

\headers{Generalization Bounds for Equivariant Networks on Markov Data}{Hui Li, Zhiguo Wang, Bohui Chen and Li Sheng}

\title{Generalization Bounds for Equivariant Networks on Markov Data\thanks{Submitted to the editors DATE.
\funding{This work was funded in part by
the National Key Research and Development Program of China under Grant
2020YFA0714003 and in part by the National Natural Science Foundation
of China under Grant 62203313.}}}

\author{
  Hui Li\thanks{College of Mathematics, Sichuan University, China \& College of Mathematics and Statistics, Sichuan University of Science and Engineering, China (\email{lihui1878293@163.com}).}
  \and
  Zhiguo Wang\thanks{College of Mathematics, Sichuan University, China (\email{wangzg315@126.com}, \email{bohui@cs.wisc.edu}, \email{lsheng@scu.edu.cn}).Corresponding author: Zhiguo Wang.}
  \and
  Bohui Chen\footnotemark[3]
  \and
  Li Sheng\footnotemark[3]  
}

\usepackage{amsopn}


%% file: article.bbl
\begin{thebibliography}{10}

\bibitem{atz2021geometric}
{\sc K.~Atz, F.~Grisoni, and G.~Schneider}, {\em Geometric deep learning on molecular representations}, Nature Machine Intelligence, 3 (2021), pp.~1023--1032.

\bibitem{bartlett2017spectrally}
{\sc P.~L. Bartlett, D.~J. Foster, and M.~J. Telgarsky}, {\em Spectrally-normalized margin bounds for neural networks}, Advances in Neural Information Processing Systems, 30 (2017).

\bibitem{behboodi2022pac}
{\sc A.~Behboodi, G.~Cesa, and T.~S. Cohen}, {\em A pac-bayesian generalization bound for equivariant networks}, Advances in Neural Information Processing Systems, 35 (2022), pp.~5654--5668.

\bibitem{bietti2019group}
{\sc A.~Bietti and J.~Mairal}, {\em Group invariance, stability to deformations, and complexity of deep convolutional representations}, Journal of Machine Learning Research, 20 (2019), pp.~1--49.

\bibitem{bogatskiy2020lorentz}
{\sc A.~Bogatskiy, B.~Anderson, J.~Offermann, M.~Roussi, D.~Miller, and R.~Kondor}, {\em Lorentz group equivariant neural network for particle physics}, in International Conference on Machine Learning, PMLR, 2020, pp.~992--1002.

\bibitem{brocker2013representations}
{\sc T.~Br{\"o}cker and T.~Tom~Dieck}, {\em Representations of compact Lie groups}, vol.~98, Springer Science \& Business Media, 2013.

\bibitem{cesa2022program}
{\sc G.~Cesa, L.~Lang, and M.~Weiler}, {\em A program to build e(n)-equivariant steerable cnns}, in International Conference on Learning Representations, 2022.

\bibitem{chen2024deep}
{\sc Q.~Chen, Z.~Li, and L.~M. Lui}, {\em A deep learning framework for diffeomorphic mapping problems via quasi-conformal geometry applied to imaging}, SIAM Journal on Imaging Sciences, 17 (2024), pp.~501--539.

\bibitem{cohen2018spherical}
{\sc T.~S. Cohen, M.~Geiger, J.~K{\"o}hler, and M.~Welling}, {\em Spherical cnns}, in International Conference on Learning Representations, 2018.

\bibitem{cohen2022steerable}
{\sc T.~S. Cohen and M.~Welling}, {\em Steerable cnns}, in International Conference on Learning Representations, 2022.

\bibitem{elesedy2021provably1}
{\sc B.~Elesedy}, {\em Provably strict generalisation benefit for invariance in kernel methods}, Advances In Neural Information Processing Systems, 34 (2021), pp.~17273--17283.

\bibitem{elesedy2022group}
{\sc B.~Elesedy}, {\em Group symmetry in pac learning}, in ICLR 2022 Workshop on Geometrical and Topological Representation Learning, 2022.

\bibitem{elesedy2021provably}
{\sc B.~Elesedy and S.~Zaidi}, {\em Provably strict generalisation benefit for equivariant models}, in International Conference on Machine Learning, PMLR, 2021, pp.~2959--2969.

\bibitem{geiger2022e3nn}
{\sc M.~Geiger and T.~Smidt}, {\em e3nn: Euclidean neural networks}, arXiv preprint arXiv:2207.09453,  (2022).

\bibitem{goumiri2023new}
{\sc S.~Goumiri, D.~Benboudjema, and W.~Pieczynski}, {\em A new hybrid model of convolutional neural networks and hidden markov chains for image classification}, Neural Computing and Applications, 35 (2023), pp.~17987--18002.

\bibitem{he2021efficient}
{\sc L.~He, Y.~Chen, Y.~Dong, Y.~Wang, Z.~Lin, et~al.}, {\em Efficient equivariant network}, Advances in Neural Information Processing Systems, 34 (2021), pp.~5290--5302.

\bibitem{he2022neural}
{\sc L.~He, Y.~Chen, Z.~Shen, Y.~Yang, and Z.~Lin}, {\em Neural epdos: Spatially adaptive equivariant partial differential operator based networks}, in The Eleventh International Conference on Learning Representations, 2022.

\bibitem{he2021gauge}
{\sc L.~He, Y.~Dong, Y.~Wang, D.~Tao, and Z.~Lin}, {\em Gauge equivariant transformer}, Advances in Neural Information Processing Systems, 34 (2021), pp.~27331--27343.

\bibitem{he2021deep}
{\sc W.~He, A.~M. Sainju, Z.~Jiang, and D.~Yan}, {\em Deep neural network for 3d surface segmentation based on contour tree hierarchy}, in Proceedings of the 2021 SIAM International Conference on Data Mining (SDM), SIAM, 2021, pp.~253--261.

\bibitem{huang2022equivariant}
{\sc H.~Huang, D.~Wang, R.~Walters, and R.~Platt}, {\em Equivariant transporter network}, in Robotics: Science and Systems, 2022.

\bibitem{jia2023seil}
{\sc M.~Jia, D.~Wang, G.~Su, D.~Klee, X.~Zhu, R.~Walters, and R.~Platt}, {\em Seil: simulation-augmented equivariant imitation learning}, in 2023 IEEE International Conference on Robotics and Automation (ICRA), IEEE, 2023, pp.~1845--1851.

\bibitem{koltchinskii2002empirical}
{\sc V.~Koltchinskii and D.~Panchenko}, {\em Empirical margin distributions and bounding the generalization error of combined classifiers}, The Annals of Statistics, 30 (2002), pp.~1--50.

\bibitem{lecun2015deep}
{\sc Y.~LeCun, Y.~Bengio, and G.~Hinton}, {\em Deep learning}, Nature, 521 (2015), pp.~436--444.

\bibitem{li2024affine}
{\sc Y.~Li, Y.~Qiu, Y.~Chen, L.~He, and Z.~Lin}, {\em Affine equivariant networks based on differential invariants}, in Proceedings of the IEEE/CVF Conference on Computer Vision and Pattern Recognition, 2024, pp.~5546--5556.

\bibitem{lyle2020benefits}
{\sc C.~Lyle, M.~Van Der~Wilk, M.~Kwiatkowska, Y.~Gal, and B.~Bloem-Reddy}, {\em On the benefits of invariance in neural networks}, arXiv preprint arXiv:2005.00178,  (2020).

\bibitem{mohri2018foundations}
{\sc M.~Mohri, A.~Rostamizadeh, and A.~Talwalkar}, {\em Foundations of machine learning}, MIT press, 2018.

\bibitem{neyshabur2017pac}
{\sc B.~Neyshabur, S.~Bhojanapalli, and N.~Srebro}, {\em A pac-bayesian approach to spectrally-normalized margin bounds for neural networks}, International Conference on Learning Representations,  (2018).

\bibitem{nguyen2023equivariant}
{\sc H.~H. Nguyen, A.~Baisero, D.~Klee, D.~Wang, R.~Platt, and C.~Amato}, {\em Equivariant reinforcement learning under partial observability}, in Conference on Robot Learning, PMLR, 2023, pp.~3309--3320.

\bibitem{pan2023tax}
{\sc C.~Pan, B.~Okorn, H.~Zhang, B.~Eisner, and D.~Held}, {\em Tax-pose: Task-specific cross-pose estimation for robot manipulation}, in Conference on Robot Learning, PMLR, 2023, pp.~1783--1792.

\bibitem{paulin2015concentration}
{\sc D.~Paulin}, {\em Concentration inequalities for markov chains by marton couplings and spectral methods}, Electron. J. Probab, 20 (2015), pp.~1--32.

\bibitem{pisier1981remarques}
{\sc G.~Pisier}, {\em Remarques sur un r{\'e}sultat non publi{\'e} de b. maurey}, S{\'e}minaire d'Analyse fonctionnelle (dit" Maurey-Schwartz"),  (1981), pp.~1--12.

\bibitem{rudolf2011explicit}
{\sc D.~Rudolf}, {\em Explicit error bounds for markov chain monte carlo}, Dissertationes Mathematicae, 485 (2012), pp.~1--93.

\bibitem{scarabosio2022deep}
{\sc L.~Scarabosio}, {\em Deep neural network surrogates for nonsmooth quantities of interest in shape uncertainty quantification}, SIAM/ASA Journal on Uncertainty Quantification, 10 (2022), pp.~975--1011.

\bibitem{shen2024efficient}
{\sc Z.~Shen, Y.~Qiu, J.~Liu, L.~He, and Z.~Lin}, {\em Efficient learning of scale-adaptive nearly affine invariant networks}, Neural Networks, 174 (2024), p.~106229.

\bibitem{simeonov2022neural}
{\sc A.~Simeonov, Y.~Du, A.~Tagliasacchi, J.~B. Tenenbaum, A.~Rodriguez, P.~Agrawal, and V.~Sitzmann}, {\em Neural descriptor fields: {SE(3)}-equivariant object representations for manipulation}, in 2022 International Conference on Robotics and Automation (ICRA), IEEE, 2022, pp.~6394--6400.

\bibitem{smets2023pde}
{\sc B.~M. Smets, J.~Portegies, E.~J. Bekkers, and R.~Duits}, {\em Pde-based group equivariant convolutional neural networks}, Journal of Mathematical Imaging and Vision, 65 (2023), pp.~209--239.

\bibitem{sokolic2017generalization}
{\sc J.~Sokolic, R.~Giryes, G.~Sapiro, and M.~Rodrigues}, {\em Generalization error of invariant classifiers}, in Artificial Intelligence and Statistics, PMLR, 2017, pp.~1094--1103.

\bibitem{truong2022generalization}
{\sc L.~V. Truong}, {\em Generalization error bounds on deep learning with markov datasets}, Advances in Neural Information Processing Systems, 35 (2022), pp.~23452--23462.

\bibitem{van2020mdp}
{\sc E.~Van~der Pol, D.~Worrall, H.~van Hoof, F.~Oliehoek, and M.~Welling}, {\em Mdp homomorphic networks: Group symmetries in reinforcement learning}, Advances in Neural Information Processing Systems, 33 (2020), pp.~4199--4210.

\bibitem{vapnik1998statistical}
{\sc V.~Vapnik}, {\em Statistical learning theory}, John Wiley \& Sons Google Schola, 2 (1998), pp.~831--842.

\bibitem{vlavcic2021affine}
{\sc V.~Vla{\v{c}}i{\'c} and H.~B{\"o}lcskei}, {\em Affine symmetries and neural network identifiability}, Advances in Mathematics, 376 (2021), p.~107485.

\bibitem{wang2022so}
{\sc D.~Wang, R.~Walters, and R.~Platt}, {\em {SO(2)}-equivariant reinforcement learning}, in International Conference on Learning Representations, 2022.

\bibitem{wang2020incorporating}
{\sc R.~Wang, R.~Walters, and R.~Yu}, {\em Incorporating symmetry into deep dynamics models for improved generalization}, International Conference on Learning Representations,  (2021).

\bibitem{weiler2019general}
{\sc M.~Weiler and G.~Cesa}, {\em General {E(2)}-equivariant steerable cnns}, Advances in Neural Information Processing Systems, 32 (2019).

\bibitem{xiao2022stability}
{\sc J.~Xiao, Y.~Fan, R.~Sun, J.~Wang, and Z.-Q. Luo}, {\em Stability analysis and generalization bounds of adversarial training}, Advances in Neural Information Processing Systems, 35 (2022), pp.~15446--15459.

\bibitem{xiao2023pac}
{\sc J.~Xiao, R.~Sun, and Z.-Q. Luo}, {\em Pac-bayesian spectrally-normalized bounds for adversarially robust generalization}, Advances in Neural Information Processing Systems, 36 (2023), pp.~36305--36323.

\bibitem{zhu2021understanding}
{\sc S.~Zhu, B.~An, and F.~Huang}, {\em Understanding the generalization benefit of model invariance from a data perspective}, Advances in Neural Information Processing Systems, 34 (2021), pp.~4328--4341.

\bibitem{zhu2022sample}
{\sc X.~Zhu, D.~Wang, O.~Biza, G.~Su, R.~Walters, and R.~Platt}, {\em Sample efficient grasp learning using equivariant models}, in Robotics: Science and Systems, 2022.

\end{thebibliography}
